\documentclass{article}

\PassOptionsToPackage{numbers, compress}{natbib}

\usepackage[final]{neurips_2021}

\usepackage[utf8]{inputenc} 
\usepackage[T1]{fontenc}    
\usepackage{url}            
\usepackage{booktabs}       
\usepackage{amsfonts}       
\usepackage{nicefrac}       
\usepackage{microtype}      
\usepackage{xcolor}         

\usepackage[compact]{titlesec}
\titlespacing*{\section}{0pt}{*0.1}{*0.1}
\titlespacing*{\subsection}{0pt}{*0.1}{*0.1}

\usepackage{enumitem}

\usepackage[colorlinks,
linkcolor=red,
anchorcolor=blue,
citecolor=blue
]{hyperref}

\usepackage{mylatexstyle}

\usepackage{xcolor}

\title{\huge Risk Bounds for Over-parameterized Maximum Margin Classification on Sub-Gaussian Mixtures}

\author{%
  Yuan Cao \\
  Department of Statistics \& Actuarial Science \\
  Department of Mathematics \\
  The University of Hong Kong\\
  \texttt{yuancao@hku.hk} \\
  \And
  Quanquan Gu\\
  Department of Computer Science\\
  University of California, Los Angeles\\
  Los Angeles, CA 90095, USA\\
  \texttt{qgu@cs.ucla.edu} \\
  \And
  Mikhail Belkin \\
  Halicioğlu Data Science Institute\\
  University of California San Diego\\
  La Jolla, CA 92093, USA\\
  \texttt{mbelkin@ucsd.edu} \\
}

\date{}

\def\cN{\mathcal{N}}

\def\dd{\mathrm{d}}

\newcommand{\la}{\langle}
\newcommand{\ra}{\rangle}

\def\CC{}

\renewcommand{\citet}{\citep}

\begin{document}

\maketitle

\begin{abstract}
    Modern machine learning systems such as deep neural networks are often highly over-parameterized so that they can fit the noisy training data exactly, yet they can still achieve small test errors in practice. In this paper, we study this ``benign overfitting''  phenomenon of the maximum margin classifier for linear classification problems. Specifically, we consider data generated from sub-Gaussian mixtures, and provide a tight risk bound for the maximum margin linear classifier in the over-parameterized setting. Our results precisely characterize the condition under which benign overfitting can occur in linear classification problems, and improve on previous work. They also have direct implications for over-parameterized logistic regression. 
\end{abstract}



\section{Introduction}

In modern machine learning, complex models such as deep neural networks have become increasingly  popular. These complicated models are  capable of  fitting noisy training data sets, while at the same time achieving small test errors. In fact, this \textit{benign overfitting} phenomenon is not a unique feature of deep learning. Even for kernel methods and linear models, \citet{belkin2018understand}  demonstrated that interpolators on the noisy training data can still perform near optimally on the test data. A series of recent works \citep{belkin2019two,muthukumar2020harmless,hastie2019surprises,bartlett2020benign}   theoretically studied how over-parameterization can achieve small population risk.


 
 In particular in \citet{bartlett2020benign} the authors considered the setting where the data are generated from a ground-truth linear model with noise, 
and established a tight population risk bound for the minimum norm linear interpolator with a matching lower bound. More recently, \citet{tsigler2020benign} further studied benign overfitting in ridge regression, and established non-asymptotic generalization bounds for over-parametrized ridge regression. They showed that those bounds are tight for a range of regularization parameter values. Notably, these results cover arbitrary covariance structure of the data, and give a nice characterization of how the spectrum of the data covariance matrix affects the population risk in the over-parameterized regime.



Very recently, benign overfitting has also been studied in the setting of linear classification \citep{chatterji2020finite,muthukumar2020classification,wang2020benign}. Specifically, \citet{muthukumar2020classification} studied the setting where the data inputs are Gaussian and the labels are generated from a ground truth linear model with label flipping noise, and showed equivalence between the hard-margin support vector machine (SVM) solution and the minimum norm interpolator to study benign overfitting. 
\citet{chatterji2020finite,wang2020benign} studied the benign overfitting phenomenon in sub-Gaussian/Gaussian mixture models and established population risk bounds for the maximum margin classifier. \citet{chatterji2020finite} leveraged the \textit{implicit bias} of gradient descent for logistic regression \citep{soudry2017implicit} to establish the risk bound. \citet{wang2020benign} established an equivalence result between classification and regression for isotropic Gaussian mixture models. While these results have offered valuable insights into the benign overfitting phenomenon for (sub-)Gaussian mixture classification, they still have certain limitations. Unlike the results in the regression setting where the eigenvalues of the data covariance matrix play a key role, \CC{the current results for Gaussian/sub-Gaussian mixture models do not show the impact of the spectrum of the data covariance matrix on the risk.}

In this paper, we study the benign overfitting phenomenon in a general sub-Gaussian mixture model that covers both the isotropic and anisotropic settings, where the $d$-dimensional features from two classes have the same covariance matrix $\bSigma$ but have different means $\bmu$ and $-\bmu$ respectively. 
We consider the over-parameterized setting where $d$ is larger than the sample size $n$, and prove a risk bound for the maximum margin classifier. 
We show that under certain conditions on 
eigenvalues of $\bSigma$, the mean vector $\bmu$ and the sample size $n$, the maximum margin classifier for this problem is identical to the minimum norm interpolator. We then utilize this result to establish a tight population risk bound of the maximum margin classifier. Our result reveals how the eigenvalues of the  covariance matrix $\bSigma$ affect the benign property of the classification problem, and is tighter and more general than existing results on sub-Gaussian/Gaussian mixture models. The contributions of this paper are as follows: 
\begin{itemize}[leftmargin = *]
    \item We establish a tight population risk bound for the maximum margin classifier. Our bound works for both the isotropic and anisotropic settings, which is more general than existing results in \citet{chatterji2020finite,wang2020benign}. When reducing our bound to the setting studied in \citet{chatterji2020finite}, our result gives a bound $\exp(-\Omega( n \| \bmu \|_2^4 / d ))$, where $n$ is the training sample size. Our bound is tighter than the risk bound $\exp(-\Omega( \| \bmu \|_2^4 / d ))$ in \citet{chatterji2020finite} by a factor of $n$ in the exponent. Our result also gives a tighter risk bound than that in \citet{wang2020benign}\footnote{After we posted our paper on arXiv, we noticed that the authors of \citet{wang2020benign} updated their paper a week later to include the results for mixture of anisotropic Gaussians and sharpened their bounds. Our results are still more general as we covered the mixture of anisotropic sub-Gaussian, and we also proved a matching lower bound.} in the so-called ``low SNR setting'': our result suggests that $ \| \bmu \|_2^4 = \omega(d / n)$ suffices to ensure an $o(1)$ population risk, while \citet{wang2020benign} requires $ \| \bmu \|_2^4 = \omega( (d / n)^{3/2})$.
    \item \CC{We establish population risk lower bounds achieved by the maximum margin classifier under two different settings. In both settings, the lower bounds match our population risk upper bound up to some absolute constants. This suggests that our population risk bound is tight.}
    \item Our analysis reveals that for a class of high-dimensional anisotropic sub-Gaussian mixture models, the maximum margin linear classifier on the training data can achieve small population risk under mild assumptions on the sample size $n$ and mean vector $\bmu$. Specifically, suppose that the eigenvalues of $\bSigma$ are $\{\lambda_k = k^{-\alpha}\}_{k=1}^d$ for some parameter $\alpha \in [0,1)$, and treat the sample size $n$ as a constant. Then our result shows that to achieve $o(1)$ population risk, the following conditions on $\| \bmu \|_2$ suffice:
    \CC{\begin{align*}
        \| \bmu \|_2 = \left\{ 
        \begin{array}{ll}
             \omega( d^{1/4 - \alpha/2}  ), & \text{ if }\alpha \in [0, 1/2),\\
             \omega( (\log(d) )^{1/4} ), & \text{ if }\alpha =  1/2, \\
             \omega( 1 ).  & \text{ if }\alpha \in ( 1/2,  1).
        \end{array}
        \right.
    \end{align*}}
    More specifically, when $\alpha = 1/2$, the condition on the mean vector $\bmu$ only has a logarithmic dependency on the dimension $d$, and when $\alpha \in (1/2,1)$, the condition on $\bmu$ for benign overfitting is dimension free.
    \item Our proof of the population risk bound introduces some tight intermediate results, which may be of independent interest. 
    Specifically, our proof utilizes the polarization identity to establish equivalence between the maximum margin classifier and the minimum norm interpolator. This is, to the best of our knowledge, the first equivalence result between classification and regression for anisotropic sub-Gaussian mixture models.
\end{itemize}



\paragraph{Additional Related Work} Our study is closely related to the phenomenon of double descent studied in recent works. \citet{belkin2019reconciling,belkin2019two} showed experimental results and provided theoretical analyses on some specific models to  demonstrate that the risk curve versus over-parameterization has a double descent shape. These results can therefore indicate that over-parameterization can be beneficial to achieve small test risk. \citet{hastie2019surprises,wu2020optimal} studied the double descent phenomenon in linear regression under the setting where the dimension $d$ and sample size $n$ can grow simultaneously but have a fixed ratio, and showed that the population risk exhibits a double descent curve with respect to the ratio. More recently, \citet{mei2019generalization,liao2020random,montanari2020interpolation} further extended the setting to random feature models and studied double descent when the sample size, data dimension and the number of random features have fixed ratios.

Our work is also related to the studies of implicit bias, which analyze the impact of training algorithms when the over-parameterized models have multiple global minima. Specifically, \citet{soudry2017implicit} showed that if the training data are linearly separable, then gradient descent on unregularized logistic regression converges directionally to the maximum margin linear classifier on the training data set. \citet{ji2019implicit} further studied the implicit bias of gradient descent for logistic regression on non-separable data. \citet{gunasekar2018characterizing} studied the implicit bias of various optimization methods for generic objective functions. \citet{gunasekar2017implicit,arora2019implicit} established implicit bias results for matrix factorization problems. More recently, \citet{lyu2019gradient} showed that gradient flow for learning homogeneous neural networks with logistic loss maximizes the normalized margin on the training data set. These studies of implicit bias offer a handle for us to connect the over-parameterized logistic regression  with the maximum margin classifiers for linear models.

\section{Problem Setting and Notation}\label{section:problemsetting}

\paragraph{Notations.} We use lower case letters to denote scalars, and use lower/upper case bold face letters to denote vectors/matrices respectively. For a vector $\vb$, we denote by $\| \vb \|_2$ the $\ell_2$-norm of $\vb$. For a matrix $\Ab$, we use $\| \Ab \|_2$, $\| \Ab \|_F$ to denote its spectral norm and Frobinuous norm respectively, and use $\tr(\Ab)$ to denote its trace. For a vector $\vb \in \RR^d$ and a positive definite matrix $\Ab$, we define $\| \vb \|_{\Ab} = \sqrt{ \vb^\top  \Ab  \vb}$. For an integer $n$, we denote $[n] = \{ 1,2, \ldots, n \}$.

We also use standard asymptotic notations $O(\cdot )$, $\Omega(\cdot )$,  $o(\cdot )$, and $\omega(\cdot )$. Let $\{a_n\}$ and $\{b_n\}$ be two sequences. If there exists a constant $C> 0$ such that $|a_n|\leq C |b_n|$ for all large enough $n$, then we denote $a_n = O(b_n)$. We denote $a_n = \Omega(b_n)$ if $b_n = O(a_n)$. Moreover, we write  $a_n = o(b_n)$ if $\lim | a_n / b_n | = 0$ and $a_n = \omega(b_n)$ if $\lim | a_n / b_n | = \infty$. We also use $\tilde O(\cdot)$ and $\tilde \Omega(\cdot)$ to hide some logarithmic terms in Big-O and Big-Omega notations.

At last, for a random variable $Z$, we denote by $\| Z \|_{\psi_2}$ and $\| Z \|_{\psi_1}$ the sub-Gaussian and sub-exponential norms of $Z$ respectively. 

\paragraph{Sub-Gaussian Mixture Model.}
We consider a model where the feature vectors are generated from a mixture of two sub-Gaussian distributions with means $\bmu$ and $-\bmu$ and the same covariance matrix $\bSigma$.
We assume that each data pair $(\xb,y)$ are generated independently from the following procedure:

\begin{enumerate}[leftmargin = *]
    \item The label $y\in\{+1,-1\}$ is generated as a Rademacher random variable.
    \item A random vector $\ub \in \RR^d$ is generated from a distribution such that the entries of $\ub $ are independent sub-Gaussian random variables with $\EE [u_j] = 0$, $\EE [u_j^2] = 1$ and $\| u_j \|_{\psi_2} \leq \sigma_u$ for all $j\in[d]$. 
    \item Let $\bSigma$ be a positive definite matrix with eigenvalue decomposition $\bSigma = \Vb \bLambda \Vb^\top$, where $\bLambda = \diag\{ \lambda_1,\ldots,\lambda_d \}$ and $\Vb$ is an orthonormal matrix consisting of the eigenvectors of $\bSigma$. We calculate the random vector $\qb$ based on $\ub$ as $\qb = \Vb \bLambda^{1/2} \ub$. This ensures that $\qb$ has mean zero and a covariance matrix $\bSigma$.
    \item The feature is given as $\xb = y\cdot \bmu + \qb$, where $\bmu \in \RR^d$ is a vector. Clearly, the mean of $\xb$ is $\bmu$ when $y = 1$ and is $-\bmu$ when  $y = -1$.
\end{enumerate}
We consider $n$ training data points $(\xb_i,y_i)$ generated independently from the above procedure, and denote
\begin{align*}
    \Xb = \yb \bmu^\top + \Qb,
\end{align*}
where $\Xb = [ \xb_1,\ldots, \xb_n ]^\top, \Qb = [ \qb_1,\ldots, \qb_n ]^\top \in \RR^{n \times d}$, and $\yb = [y_1,\ldots, y_n]^\top \in \{\pm1\}^n$. 
For any $\btheta\in\RR^d$, the population risk of the linear classifier $\xb \rightarrow \la \btheta, \xb\ra$ is defined as:
\begin{align*}
    R(\btheta) = \PP\big( y\cdot \la \btheta, \xb \ra < 0 \big).
\end{align*}
In this paper, we consider the maximum margin linear classifier  $\hat\btheta_{\text{SVM}}$, i.e., the solution to the hard-margin support vector machine:
\begin{align*} 
    \hat\btheta_{\text{SVM}} = \argmin \| \btheta \|_2^2, ~~\text{subject to } y_i \cdot \la\btheta , \xb_i\ra \geq 1, i\in [n],
\end{align*}
and study its population risk $R(\hat\btheta_{\text{SVM}})$.

A recent work \citep{chatterji2020finite} has studied a similar sub-Gaussian mixture model under an assumption that $\tr(\bSigma) = \Omega(d)$, and considered additional label flipping noises. In this paper, we do not introduce the label flipping noises for simplicity, but we consider a general covariance matrix $\bSigma$ to cover the general anisotropic setting. 
\CC{It is worth noting that although our model is not exactly the same as \citet{chatterji2020finite} because we don’t have additional label flipping noise, there is still noise in our model because of the nature of sub-Gaussian mixture model. For example, consider a mixture of two Gaussian distributions. The two Gaussian clusters have non-trivial overlap, and the Bayes optimal classifier has non-zero Bayes risk as long as $\|\bmu \|_2 < \infty$. Therefore, the Bayes optimal classifier and the interpolating classifier are generally quite different. 
In general, a model is appropriate for the study of benign overfitting whenever the optimal classifier has non-zero Bayes risk.}

Our model is rather general and covers the following examples. 

\begin{example}[Gaussian mixture model]
The most straight-forward example is when the data are generated from Gaussian mixtures $N(\bmu,\bSigma)$ and $N(-\bmu,\bSigma)$. This is covered by our model when the sub-Gaussian vector $\ub$ is a standard Gaussian random vector.
\end{example}

\begin{example}[Rare/weak feature model]\label{example:rare-weak}
The rare-weak model is a special case of the Gaussian mixture model where $\bSigma = \Ib$ and $\bmu$ is a sparse vector with $s$ non-zero entries equaling $\gamma$.
\end{example}

The rare/weak feature model was originally investigated by \citet{donoho2008higher,jin2009impossibility}, and was recently studied by \citet{chatterji2020finite}. 

\paragraph{Connection to Over-parameterized Logistic Regression.}
Our study of the maximum margin classifier is closely related to over-parameterized logistic regression.
In logistic regression, we consider the following empirical loss minimization problem:
\begin{align*}
    \min_{\btheta\in\RR^d} L(\btheta ) := \frac{1}{n} \sum_{i=1}^n \log [ 1 + \exp(- y_i \cdot \la \btheta, \xb_i \ra ) ].
\end{align*}
We solve the above optimization problem  with gradient descent
\begin{align}\label{eq:GDupdate}
    \btheta^{(t+1)} = \btheta^{(t)} - \eta\cdot \nabla  L(\btheta^{(t)} ),
\end{align}
where $\eta>0$ is the learning rate.

In the over-parameterized setting where $d \gg n$, it is evident that the training data points are linearly separable with high probability (for example, $\Xb\Xb^\top$ is invertible with high probability and the minimum norm interpolator $\hat\btheta_{\text{LS}} = \Xb^\top (\Xb\Xb^\top)^{-1} \yb$ separates the training data.). For linearly separable data, a series of recent works have studied the \textit{implicit bias} of (stochastic) gradient descent for logistic regression \citep{soudry2017implicit,ji2019implicit,nacson2019stochastic}. These results demonstrate that among all linear classifiers that can classify the training data correctly, gradient descent will converge to the one that maximizes the $\ell_2$ margin. Such an implicit bias result is summarized in the following lemma.
\begin{lemma}[Theorem~3 in \citet{soudry2017implicit}]\label{lemma:implicitbias}
    Suppose that the training data set $\{(\xb_i,y_i)\}$ is linearly separable. Then as long as $\eta>0$ is small enough, the gradient descent iterates $\btheta^{(t)}$ for logistic regression defined in \eqref{eq:GDupdate} has the following direction limit:
    \begin{align*}
    \lim_{t\rightarrow \infty}\frac{\btheta^{(t)}}{\| \btheta^{(t)} \|_2} =  \frac{\hat\btheta_{\text{SVM}}}{\| \hat\btheta_{\text{SVM}} \|_2},
\end{align*}
where $\hat\btheta_{\text{SVM}}$ is the maximum margin classifier.
\end{lemma}

Lemma~\ref{lemma:implicitbias} suggests that our risk bound of the maximum margin classifier $\hat\btheta_{\text{SVM}}$ directly implies a risk bound for the over-parameterized logistic regression trained by gradient descent. 

\section{Main Results}\label{section:mainresults}

\CC{In this section, we present our main result on the population risk bound of the maximum margin classifier, and then give a lower bound result to demonstrate the tightness of our upper bounds. We also showcase the application of our results to isotropic and anisotropic sub-Gaussian mixture models to study the conditions under which benign overfitting occurs.}

The main result of this paper is given in the following theorem, where we establish the population risk bound for the maximum margin classifier $R(\hat\btheta_{\text{SVM}})$.

\begin{theorem}\label{thm:BOnew}
Suppose that 
$\tr( \bSigma ) \geq C \max\big\{ n^{3/2} \|\bSigma \|_2, n\| \bSigma \|_F , n\sqrt{\log(n)}\cdot \| \bmu \|_{\bSigma} \big\}$ 
and
$\| \bmu \|_2^2 \geq C \| \bmu \|_{\bSigma}$ 
for some absolute constant $C$. Then with probability at least $1 - n^{-1}$, the maximum margin classifier $\hat\btheta_{\text{SVM}}$ has the following risk bound
\begin{align*}
    R(\hat\btheta_{\text{SVM}}) \leq \exp\Bigg(  \frac{ - C' n \| \bmu \|_2^4  }{  n  \| \bmu \|_{\bSigma}^2+ \| \bSigma\|_F^2 +  n\| \bSigma\|_2^2 } \Bigg),
\end{align*}
where $C'$ is an absolute constant.
\end{theorem}

Theorem~\ref{thm:BOnew} gives the population risk bound of the maximum margin classifier $\hat\btheta_{\text{SVM}}$. Based on the implicit bias of gradient descent for over-parameterized logistic regression (Lemma~\ref{lemma:implicitbias}), we have that the gradient descent iterates $\btheta^{(t)}$ satisfy that
\begin{align*}
    \lim_{t\rightarrow \infty} R(\btheta^{(t)}) &= \lim_{t\rightarrow \infty} R(\btheta^{(t)}/ \| \btheta^{(t)} \|_2) =R(\hat\btheta_{\text{SVM}}/\| \hat\btheta_{\text{SVM}} \|_2) = R(\hat\btheta_{\text{SVM}}).
\end{align*}
Therefore, the same risk bound in Theorem~\ref{thm:BOnew} also applies to the over-parameterized logistic regression trained by gradient descent.

\paragraph{Population Risk Lower Bound} We further present lower bounds on the population risk achieved by the maximum margin classifier, which demonstrate that our population risk upper bound in Theorem~\ref{thm:BOnew} is tight. We have the following theorem.

\begin{theorem}\label{thm:BO_lowerbound}
Consider Gaussian mixture model with covariance matrix $\bSigma$ and mean vectors $\bmu$ and $-\bmu$. 
Suppose that
    $\tr( \bSigma ) \geq C \max\big\{ n^{3/2} \|\bSigma \|_2, n\| \bSigma \|_F , n\sqrt{\log(n)}\cdot \| \bmu \|_{\bSigma} \big\}$, 
and
$\| \bmu \|_2^2 \geq C \| \bmu \|_{\bSigma}$ for some constant $C$.
Then there exist absolute constants $C',C''$, such that
the following results hold:
\begin{enumerate}[leftmargin = *]
    \item If $ n \| \bmu \|_{\bSigma}^2 \geq C (\| \bSigma \|_F^2 + n \| \bSigma\|_2^2)$, then with probability at least $1 - n^{-1}$, 
    $$ R(\hat\btheta_{\text{SVM}}) \geq C''\exp\big( - C'  \| \bmu \|_2^4  /  \| \bmu \|_{\bSigma}^2  \big).$$
    \item If $\| \bSigma \|_F^2 \geq C n (\| \bmu \|_{\bSigma}^2 + \| \bSigma\|_2^2)$, then with probability at least $1 - n^{-1}$, $$R(\hat\btheta_{\text{SVM}}) \geq C''\exp \big(  - C' n \| \bmu \|_2^4 / \| \bSigma\|_F^2  \big).$$ 
\end{enumerate}
\end{theorem}

Theorem~\ref{thm:BO_lowerbound} gives lower bounds for the population risk in two settings: (i) $ n \| \bmu \|_{\bSigma}^2 \geq C (\| \bSigma \|_F^2 + n \| \bSigma\|_2^2)$; and (ii) $\| \bSigma \|_F^2 \geq C n (\| \bmu \|_{\bSigma}^2 + \| \bSigma\|_2^2)$. Note that in the population risk upper bound in Theorem~\ref{thm:BOnew}, there are three terms in the denominator of the exponent: $\| \bmu \|_{\bSigma}^2 $, $\| \bSigma \|_F^2 $, and $n  \| \bSigma\|_2^2$. Therefore, setting (i) and setting (ii) in Theorem~\ref{thm:BO_lowerbound} correspond to the cases when the first or the second term is the leading term, respectively. Moreover, it is also easy to check that under both settings, our lower bound in Theorem~\ref{thm:BO_lowerbound} matches the upper bound in Theorem~\ref{thm:BOnew}. This suggests that our population risk bound in Theorem~\ref{thm:BOnew} is tight.

\paragraph{Implications for Specific Examples.} \CC{Theorem~\ref{thm:BOnew} holds for general covariance matrices $\bSigma$, and illustrates how the spectrum of $\bSigma$ affects the population risk of \CC{the maximum margin classifier}. This makes our result more general than the recent results in \citet{chatterji2020finite,wang2020benign}, where the population risk bounds are given only in terms of the sample size $n$, dimension $d$ and the norm of the mean vector $\|\bmu\|_2$. In fact, when we specialize our general result to the isotropic setting, our result also provides a tighter risk bound than these existing results. Specifically, our population risk bound for the isotropic setting is given in the following corollary. }



\begin{corollary}[\CC{Isotropic sub-Gaussian mixtures}]\label{col:isotropic}
Consider the setting where $\bSigma = \Ib$. 
Suppose that
$d \geq C \max\big\{ n^{2} , n\sqrt{\log(n)}\cdot \| \bmu \|_{2} \big\}$ 
and
$\| \bmu \|_2 \geq C $ 
for some absolute constant $C$. Then with probability at least $1 - n^{-1}$, the maximum margin classifier $\hat\btheta_{\text{SVM}}$ has the following risk bound
\begin{align*}
    R(\hat\btheta_{\text{SVM}} ) \leq  \exp\bigg( - \frac{  C' n \| \bmu \|_2^4  }{  n  \| \bmu \|_{2}^2+ d } \bigg),
\end{align*}
where $C'$ is an absolute constant.
\end{corollary}
\begin{remark}
 \citet{chatterji2020finite} recently gave a risk bound of order $\exp( - \Omega( \| \bmu \|_2^4 / d ))$ for sub-Gaussian mixture models under the condition that $d = \Omega( \max\{ n^2\log(n), n\| \bmu \|_2^2 \} )$. In comparison, our result in Corollary~\ref{col:isotropic} only requires the condition $d = \Omega(\max\{n^2,n\sqrt{\log(n)}\cdot \| \bmu \|_2\} )$, which is milder. Moreover, when the stronger condition $d = \Omega(n\| \bmu \|_2^2 )$ holds, our risk bound becomes $\exp( - \Omega( n \| \bmu \|_2^4 / d ))$, which is better than the result of \citet{chatterji2020finite} by a factor of $n$ in the exponent.
\end{remark}





Besides being tighter than previous results  when reduced to the isotropic setting, 
Theorem~\ref{thm:BOnew} 
covers both the isotropic and anisotropic settings. In the following, we provide some case studies under the anisotropic setting and show how the decay rate of the eigenvalues of the covariance matrix $\bSigma$ affects the population risk.

It is worth noting that the assumption of Theorem~\ref{thm:BOnew} requires that $\tr(\bSigma)$ is large enough, while the risk bound in Theorem~\ref{thm:BOnew} only depends on $\| \bSigma \|_F$ and $\| \bSigma \|_2$. In the over-parameterized setting where the dimension $d$ is large, it is possible that for certain covariance matrices $\bSigma$ with appropriate eigenvalue decay rates, $\tr(\bSigma) \gg 1$ while $\| \bSigma \|_F, \| \bSigma \|_2 = O(1) $. This implies that for many anisotropic sub-Gaussian mixture models,  the assumptions in Theorem~\ref{thm:BOnew} can be easily satisfied, while the risk bound can be small at the same time. 
Following this intuition, we study the conditions under which the the maximum margin interpolator $\hat\btheta_{\text{SVM}}$ achieves $o(1)$ population risk. We denote by $\lambda_k$ the $k$-th largest eigenvalue of $\bSigma$, and consider a polynomial decay 
spectrum $\{ \lambda_k = k^{-\alpha} \}_{k=1}^d$, where we introduce a parameter $\alpha$ to control the eigenvalue decay rate.
We have the following corollary.

\begin{corollary}[\CC{Anisotripic sub-Gaussian mixtures with polynomial spectrum decay}]\label{col:polynomialdecay}
   \CC{Suppose that $\lambda_k = k^{-\alpha}$, $n$ is a large enough constant, and one of the following conditions hold:
\begin{enumerate}[leftmargin = *]
    \item $\alpha\in [0, 1/2)$, $ d = \tilde\Omega(  ( \| \bmu \|_{\bSigma})^{\frac{1}{1 - \alpha}}  )$, and $\| \bmu \|_2 = \omega( 1 + d^{1/4 - \alpha/2}  )$.
    \item $\alpha = 1/2$, $ d = \tilde\Omega( \| \bmu \|_{\bSigma}^2  )$, and $\| \bmu \|_2 = \omega(  (\log(d))^{1/4} )$.
    \item $\alpha\in (1/2, 1 )$, $ d = \tilde\Omega(  ( \| \bmu \|_{\bSigma})^{\frac{1}{1 - \alpha}}  )$, and $\| \bmu \|_2 = \omega(1)$.
\end{enumerate}
Then with probability at least $1 - n^{-1}$, the population risk of the maximum margin classifier satisfies $R(\hat\btheta_{\text{SVM}}) = o(1)$. }
\end{corollary}


Corollary~\ref{col:polynomialdecay} follows by calculating the orders of $\tr(\bSigma) = \sum_{k=1}^d \lambda_k$ and $\| \bSigma \|_F^2 =  \sum_{k=1}^d \lambda_k^2$. Here we treat the sample size $n$ as a constant for simplicity. A full version of the corollary with detailed dependency on $n$ is given as Corollary~\ref{col:polynomialdecay_with_n} in Appendix~\ref{section:corollaryproof} together with the proof. 
Intuitively,  when $\| \bmu \|_2$ is large, the two classes are far away from each other and therefore linear classifiers can achieve small population risk. From Corollary~\ref{col:polynomialdecay}, we can see that the decay rate of the eigenvalues of the covariance matrix $\bSigma$ determines how large $\| \bmu \|_2 $ needs to be to ensure small population risk: when the $\{\lambda_k\}$ decays faster (i.e., when $\alpha$ is larger), the maximum margin classifier can achieve $o(1)$ population risk with a smaller $\| \bmu \|_2 $. 

Corollary~\ref{col:polynomialdecay} also exhibits a certain ``phase transition'' regarding the eigenvalue decay rate and the conditions on $\| \bmu \|_2$. We can see that the eigenvalue decay rate can be divided into three regimes $\alpha\in [0, 1/2)$, $\alpha =  1/2$ and $\alpha\in (1/2, 1)$. Under the condition that $ d = \tilde\Omega(  ( \| \bmu \|_{\bSigma})^{\frac{1}{1 - \alpha}}  )$, achieving $o(1)$ risk in each of these regimes requires $\| \bmu \|_2 = \omega( d^{1/4} ) $, $\| \bmu \|_2 = \omega( [\log(d)]^{1/4} ) $, and $\| \bmu \|_2 = \omega( 1) $ respectively. Specifically, when $\alpha\in (1/2, 1)$, the condition on $\bmu$ is independent of the dimension $d$. This means that when $\alpha\in (1/2, 1)$, for any $\epsilon > 0$, as long as $\| \bmu \|_2 = \Omega( \sqrt{\log(\epsilon)})$, we have
\begin{align*}
    \lim_{d\rightarrow \infty} R(\hat\btheta_{\text{SVM}}) \leq \epsilon.
\end{align*}
Therefore, our result covers the infinite dimensional setting when the eigenvalues of the covariance matrix $\bSigma$ have an appropriate decay rate, i.e., $\alpha\in (1/2, 1)$.

\CC{Another interesting observation in Theorem~\ref{thm:BOnew} is that it uses both $\| \bmu \|_{\bSigma}$ and $\| \bmu \|_2$, and therefore the alignment between $\bmu$ and the eigenvectors of $\bSigma$ can affect the population risk bound. In our discussion above, we have mainly focused on the worst case scenario where the direction of $\bmu$ aligns with the first eigenvector of $\bSigma$. In the following corollary, we discuss the case where $\bmu$ is parallel to the eigenvector of $\bSigma$ corresponding to the eigenvalue $\lambda_k$.}

\begin{corollary}[\CC{Risk bounds for $\bmu$ along different directions}]\label{col:alignment_mu_Sigma}
\CC{Suppose that $\bSigma \bmu = \lambda_k \bmu$ for some $k \in [d]$, 
$\tr( \bSigma ) \geq C \max\big\{ n^{3/2} \|\bSigma \|_2, n\| \bSigma \|_F , n\sqrt{\lambda_k \log(n)}\cdot \| \bmu \|_{2} \big\}$ 
and
\CC{$\| \bmu \|_2^2 \geq C \lambda_k$} 
for some absolute constant $C$. Then with probability at least $1 - n^{-1}$, the maximum margin classifier $\hat\btheta_{\text{SVM}}$ has the following risk bound
\begin{align*}
    R(\hat\btheta_{\text{SVM}}) \leq \exp\Bigg(  \frac{ - C' n \| \bmu \|_2^4  }{  n \lambda_k\cdot  \| \bmu \|_{2}^2+ \| \bSigma\|_F^2 +  n\| \bSigma\|_2^2 } \Bigg),
\end{align*}
where $C'$ is an absolute constant.}
\end{corollary}

\CC{We can see that, when $\bmu$ aligns with the eigendirections corresponding to a smaller eigenvalue of $\bSigma$, then Corollary~\ref{col:alignment_mu_Sigma} holds under milder conditions on $\tr(\bSigma)$ and $\| \bmu \|_2$, and the population risk achieved by the maximum margin solution is also better. This phenomenon perfectly matches the geometric intuition of sub-Gaussian mixture classifications, as is illustrated in Figure~\ref{fig:mualignment}.}

\begin{figure}[ht!]
	\begin{center}
 		\hspace{-0.2in}
			\subfigure[$\bmu $
 aligns with $\vb_2$]{\includegraphics[height=1.3in,angle=0]{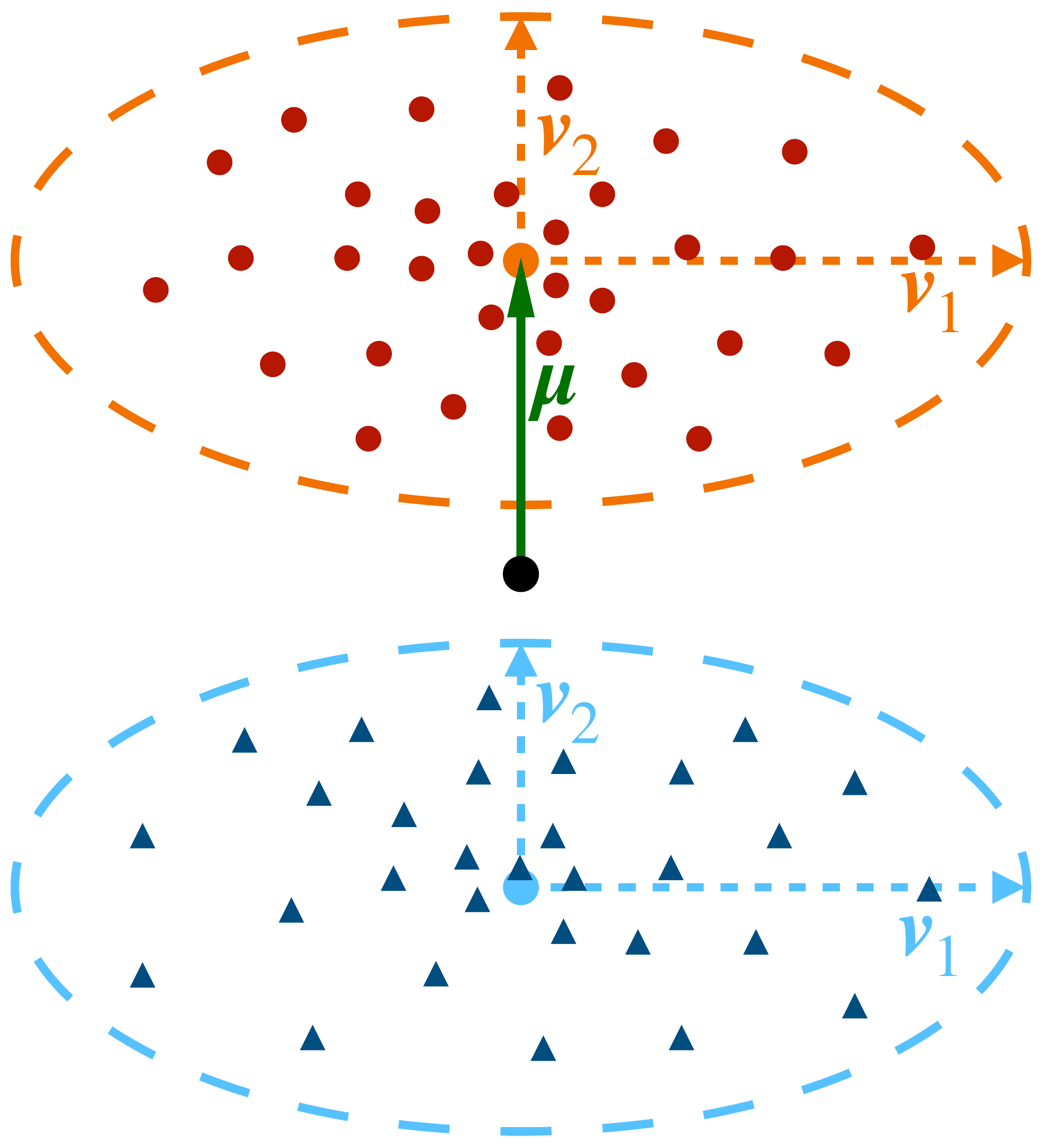}\label{subfig:v2}}\qquad
			\subfigure[$\bmu $
 points at a random direction]{\includegraphics[height=1.3in,angle=0]{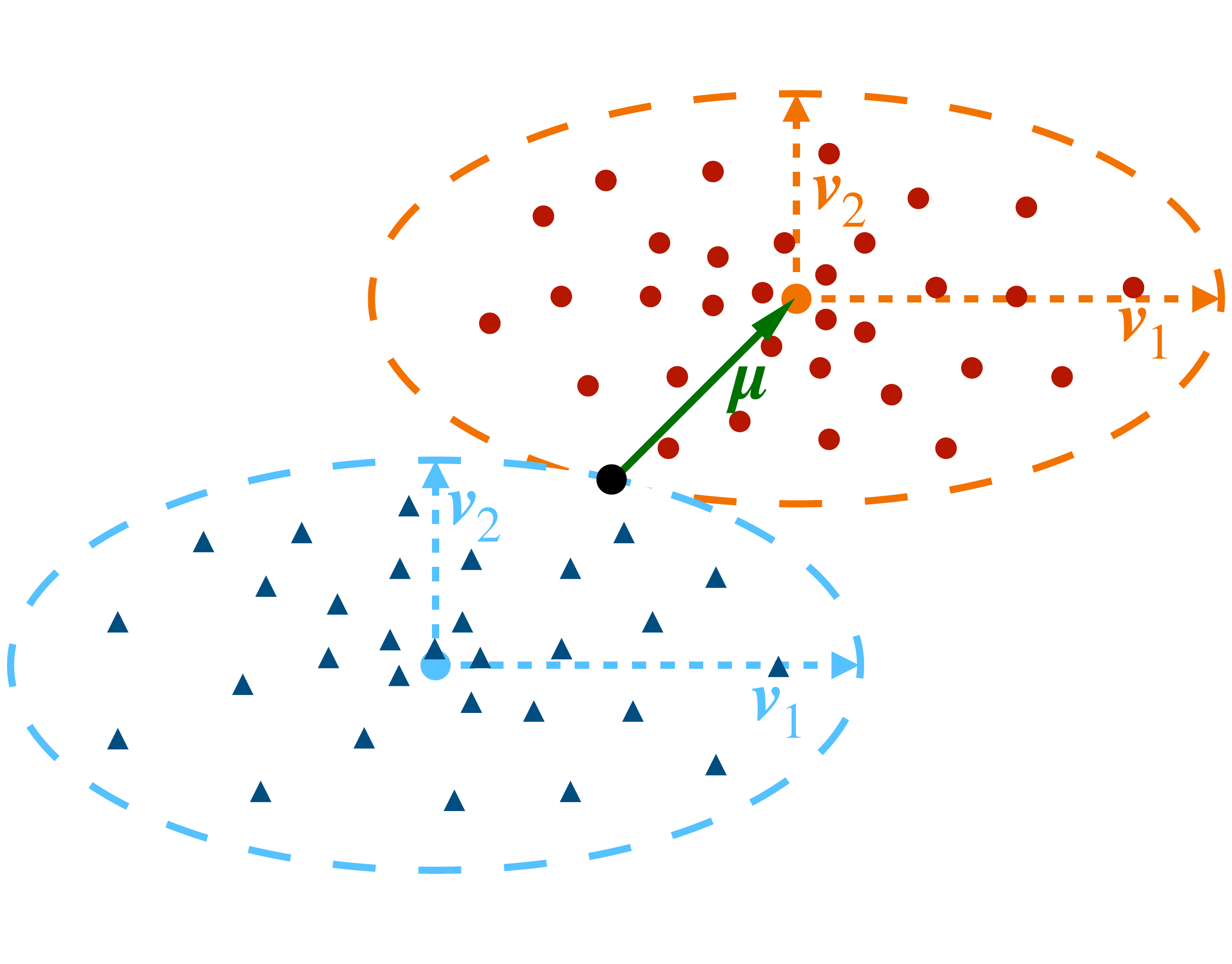}\label{subfig:mix}}\qquad
			\subfigure[$\bmu $
 aligns with $\vb_1$]{\includegraphics[height=1.3in,angle=0]{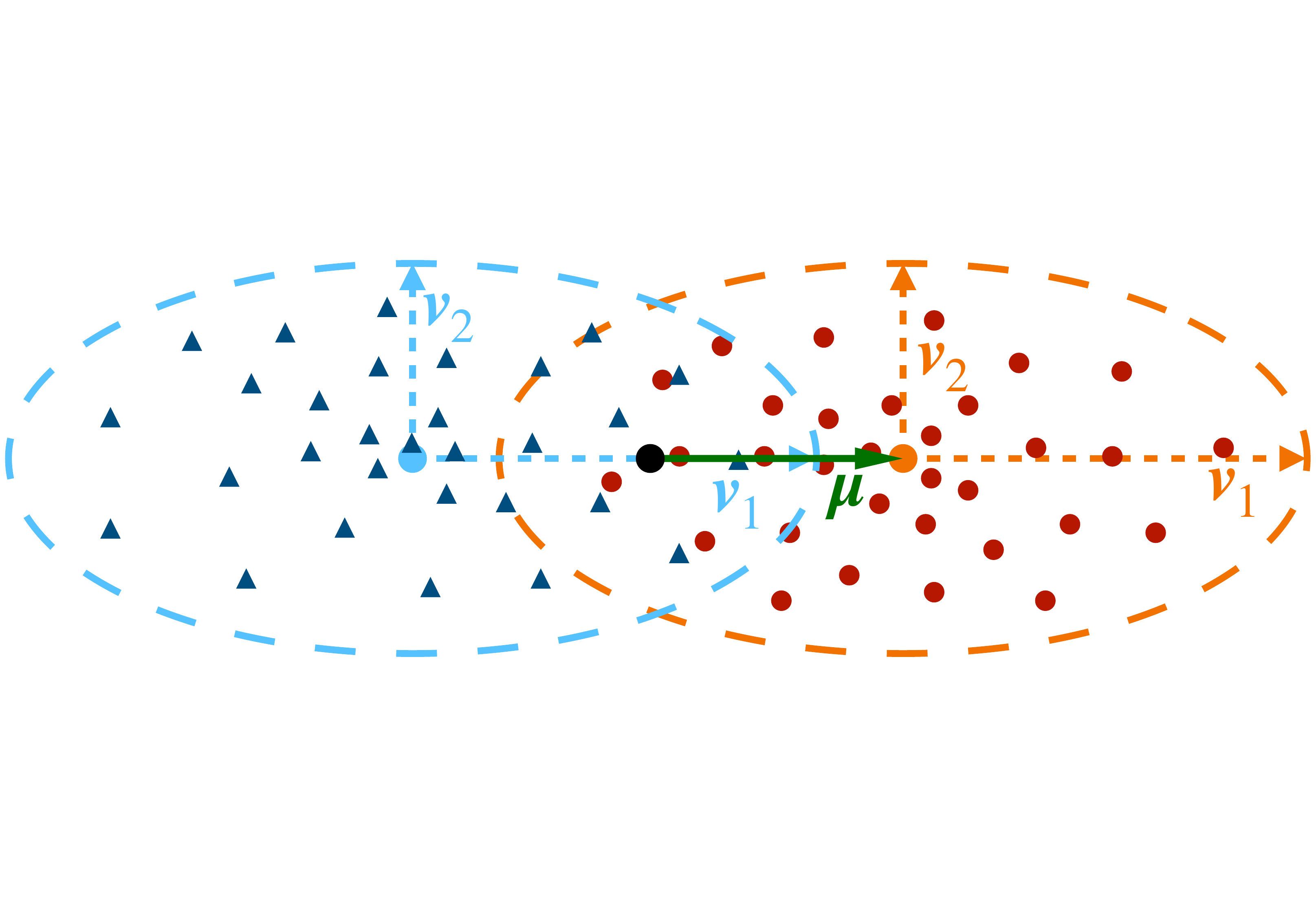}\label{subfig:v1}}
	\end{center}
	\vskip -12pt
	\caption{\CC{A $2$-dimensional illustration of  sub-Gausisan mixture classification problems with different directions of $\bmu$. We consider the setting where $\bSigma \in \RR^{2\times 2}$ has two eigenvalues  $\lambda_1 > \lambda_2$ with the corresponding eigenvectors $\vb_1,\vb_2$. (a) shows the setting where $\bmu $
 aligns with $\vb_2$. (b) shows the setting where $\bmu $
 points at a random direction. (c) is for the case when $\bmu $
 aligns with $\vb_1$. These figures clearly show that (a) is the easiest case for classification and (c) is the hardest case. This matches the result in Corollary~\ref{col:alignment_mu_Sigma}.}
	} 
	\label{fig:mualignment}
\end{figure}

At last, we can also apply our risk bound to the rare/weak feature model defined in Example~\ref{example:rare-weak}. We have the following corollary. 

\begin{corollary}[\CC{Rare/weak feature model}]\label{col:rare-weak}
Consider the rare/weak feature model (Example~\ref{example:rare-weak}). Suppose that 
$d \geq C \max\{ n^{2} , \gamma n\sqrt{ s \log(n)} \}$ 
and
$\gamma\sqrt{s} \geq C $ 
for some large enough absolute constant $C$. Then when $n$ is large enough, with probability at least $1 - n^{-1}$, the maximum margin classifier $\hat\btheta_{\text{SVM}}$ has the following risk bound
\begin{align*}
    R( \hat\btheta_{\text{SVM}} ) \leq  \exp\bigg( - \frac{  C' n \gamma^4s^2  }{  n  \gamma^2 s + d } \bigg),
\end{align*}
where $C'$ is an absolute constant.
\end{corollary}
By Corollary~\ref{col:rare-weak}, we can see that our bound is tighter by a factor of $n$ in the exponent compared with the risk bound in \citet{chatterji2020finite} for the rare/weak feature model. Under the setting where $n$ and $\gamma$ are fixed constants, our bound can also be compared with the negative result in \citet{jin2009impossibility}, which showed that achieving a small population risk is impossible when $s = O(d^2)$. Our result, on the other hand, demonstrates that when $s = \omega(d^2)$, $o(1)$ population risk is achievable.



\section{Proof of the Main Results}\label{section:proof_main}
In this section, we explain how we establish the population risk bound of the maximum margin classifier, and give the proof of Theorem~\ref{thm:BOnew}. 

For classification problems, one of the key challenges is that the maximum margin classifier usually does not have an explicit form solution. To overcome this difficulty,  \citet{chatterji2020finite} 
utilized the implicit bias results (Lemma~\ref{lemma:implicitbias}) to get a handle on the relationship between the maximum margin classifier and the training data. 
More recently, \citet{wang2020benign} showed that for isotropic Gaussian mixture models, an explicit form of $\hat\btheta_{\text{SVM}}$ can be calculated by the equivalence between hard-margin support vector machine and minimum norm least square regression. Notably, it was shown that such an equivalence result holds under the assumptions of \citet{chatterji2020finite} and no any additional assumptions are needed. In this paper, we also study the equivalence between classification and regression as a first step. However, our proof works for a more general setting that covers both isotropic and anisotropic sub-Gaussian mixtures, and introduces a novel proof technique based on the polarization identity that leads to a tighter bound. 
We present this result in Section~\ref{section:equivalence}.


\paragraph{Step 1. Equivalence Between Classification and Regression.}\label{section:equivalence}
Here we establish an equivalence guarantee for the 
maximum margin classifier
and the minimum norm interpolator. Note that the definitions of the minimum norm interpolator $\hat\btheta_{\text{LS}} $ and the maximum margin classifier $\hat\btheta_{\text{SVM}}$ are as follows:
\begin{align*}
    &\hat\btheta_{\text{LS}} := \argmin \| \btheta \|_2^2, ~~\text{subject to } y_i \cdot \la\btheta , \xb_i\ra = 1, i\in [n],\\
    &\hat\btheta_{\text{SVM}} = \argmin \| \btheta \|_2^2, ~~\text{subject to } y_i \cdot \la\btheta , \xb_i\ra \geq 1, i\in [n].
\end{align*}
We can see that the two optimization problems have the same solution when all the training data are support vectors, i.e., all the inequalities become equalities in the constraints \citep{muthukumar2020classification,hsu2020proliferation}. Here we derive such an equivalence result for sub-Gaussian mixture models. The result is as follows.


\begin{proposition}\label{prop:interpolationregression}
Suppose that 
$\tr( \bSigma ) \geq C \max\{ n^{3/2} \| \bSigma \|_2, n \| \bSigma \|_F, n\sqrt{\log(n)}\cdot \| \bmu \|_{\bSigma} \} $ 
for some absolute constant $C$. Then with probability at least $ 1 - O(n^{-2}) $, 
$\hat\btheta_{\text{SVM}} =  \hat\btheta_{\text{LS}}$.
\end{proposition}
The proof of Proposition~\ref{prop:interpolationregression} utilizes an argument based on the polarization identity to give a tight bound, which may be of independent interest. The details are given in Appendix~\ref{prop:interpolationregression}.





\paragraph{Step 2. Population Risk of the Maximum Margin Classifier.}
We derive the population risk bound for the maximum margin classifier and provide the proof of Theorem~\ref{thm:BOnew}. We first present the following lemma on the risk bound of linear classifiers for sub-Gaussian mixture models.

\begin{lemma}\label{lemma:subGaussian_riskbound}
There exists an absolute  constant $C$ such that, for any $\btheta\in \RR^d$, the following risk bound holds:
\begin{align*}
    R(\btheta) \leq \exp\bigg( - \frac{ C( \btheta^\top \bmu )^2 }{ \|\btheta\|_{\bSigma}^2} \bigg).
\end{align*}
A similar result is given in \citet{chatterji2020finite} where $\bSigma$ is replaced by $\Ib$. Our result here depends on the full spectrum of the covariance matrix and is sharper than \citet{chatterji2020finite} when $\bSigma$ has decaying eigenvalues.

\end{lemma}
The proof of Lemma~\ref{lemma:subGaussian_riskbound} is given in Appendix~\ref{section:proof_appendix1}. 
In addition to this risk bound for general vector $\btheta$, we also have the following explicit calculation for $\hat\btheta_{\text{SVM}}$ thanks to our analysis in Section~\ref{section:equivalence}. This is because the minimum norm interpolator $\hat\btheta_{\text{LS}}$ has the explicit form $\hat\btheta_{\text{LS}} = \Xb^\top (\Xb\Xb^\top)^{-1} \yb$.
Therefore by Proposition~\ref{prop:interpolationregression}, we also have
    $\hat\btheta_{\text{SVM}} = \Xb^\top (\Xb\Xb^\top)^{-1} \yb$. 
Plugging this calculation into the risk bound in Lemma~\ref{lemma:subGaussian_riskbound} and utilizing the model definition $\Xb = \yb \bmu^\top + \Qb$, we are able to show the following risk bound for $\hat\btheta_{\text{SVM}}$.
\begin{lemma}\label{lemma:thetaSVM_riskbound} 
    Suppose that 
$\tr( \bSigma ) \geq C \max\{n\sqrt{\log(n)}, n^{3/2} \| \bSigma \|_2, n \| \bSigma \|_F, n \| \bmu \|_{\bSigma} \} $ 
 for some absolute constant $C$. Then with probability at least $ 1 - O( n^{-2} )$, 
\begin{align*}
    &R(\hat\btheta_{\text{SVM}} )\leq \exp\bigg\{ \frac{ - C'\cdot [ \yb^\top (\Xb\Xb^\top)^{-1} \Xb \bmu  ]^2 }{(\yb^\top (\Xb\Xb^\top)^{-1} \yb)^2 \| \bmu \|_{\bSigma}^2 + \| \Qb^\top (\Xb\Xb^\top)^{-1} \yb \|_{\bSigma}^2 } \bigg\},
\end{align*}
where $C'$ is an absolute constant.
\end{lemma}
Lemma~\ref{lemma:thetaSVM_riskbound} 
utilizes the structure of the model to divide the denominator in the exponent into two terms. 
Motivated by this result, we define
\begin{align*}
    I_1 = [ \yb^\top (\Xb\Xb^\top)^{-1} \Xb \bmu  ]^2,~~I_2 = (\yb^\top (\Xb\Xb^\top)^{-1} \yb)^2 \cdot \| \bmu \|_{\bSigma}^2,~~I_3 =  \| \Qb^\top (\Xb\Xb^\top)^{-1} \yb \|_{\bSigma}^2.
\end{align*}
This leads to our analysis in the next step.

\paragraph{Step 3. Bounds for $I_1$, $I_2$ and $I_3$.} In the following, we develop a lower bound for $I_1$ and upper bounds for $I_2$ and $I_3$ respectively. The following lemma summarizes the bounds.


\begin{lemma}\label{lemma:I1I2I3bounds}
    Suppose that
$\tr( \bSigma ) \geq C \max\{ n, n \|\bSigma \|_2, \sqrt{n} \|\bSigma \|_F , n \| \bmu \|_{\bSigma} \}$ 
and
$\| \bmu \|_2^2 \geq C \| \bmu \|_{\bSigma}$ 
for some absolute constant $C$. Then when $n$ is large enough, with probability at least $1 - O(n^{-2})$, 
\begin{align*}
    &I_1 \geq C'^{-1} H(\bmu, \Qb, \yb, \bSigma)\cdot n^2 \cdot   \| \bmu \|_2^4,\\
    &I_2 \leq C' H(\bmu, \Qb, \yb, \bSigma)\cdot n^2 \cdot \| \bmu \|_{\bSigma}^2,\\
    &I_3 \leq C' H(\bmu, \Qb, \yb, \bSigma)\cdot (  n\cdot \|\bSigma\|_F^2 +  n^2\cdot \| \bSigma\|_2^2),
\end{align*}
where $H(\bmu, \Qb, \yb, \bSigma) > 0$ is a strictly positive coefficient, and $C' > 0$ is an absolute constant.  
\end{lemma}

The proof of Lemma~\ref{lemma:I1I2I3bounds} is given in Appendix~\ref{section:proof_appendix1}. 
To illustrate the key idea in the proof of Lemma~\ref{lemma:I1I2I3bounds}, we take $I_3$ as an example. 
Based on our model in Section~\ref{section:problemsetting}, we have $\Qb = \Zb \bLambda^{1/2} \Vb^\top$, where $\Zb\in\RR^{n\times d}$ is a random matrix with independent sub-Gaussian entries, and $\bLambda$, $\Vb$ are defined based on the eigenvalue decomposition $\bSigma = \Vb \bLambda \Vb^\top$. By some linear algebra calculation (see the proof of Lemma~\ref{lemma:I1I2I3bounds} in Appendix~\ref{section:proof_appendix1} for more details), we have
\begin{align}\label{eq:main_lemmaproof_eq2}
    I_3 = \ab^\top (\Zb\bLambda \Zb^\top)^{-1} \Zb\bLambda^2 \Zb^\top (\Zb\bLambda \Zb^\top)^{-1} \ab,
\end{align}
where  $\| \ab \|_2^2 = O(D^{-2} n) $ with
\begin{align*}
    D &=  \yb^\top (\Qb \Qb^\top)^{-1} \yb \cdot  (\| \bmu\|_2^2 - \bmu^\top \Qb^\top (\Qb \Qb^\top)^{-1} \Qb\bmu)  + (1 + \yb^\top(\Qb \Qb^\top)^{-1} \Qb\bmu)^2.
\end{align*}
The key observation here is that while the term $D$ above has a very complicated form, it is not necessary to bound it. This is because $D^{-2}$ is a common term that appears in all $I_1$, $I_2$ and $I_3$ and therefore can be canceled out when calculating the ratio $I_1 / (I_2 + I_3)$. 
With the calculation in \eqref{eq:main_lemmaproof_eq2}, we are able to invoke the following eigenvalue concentration inequalities (see Lemma~\ref{lemma:eigenvalue_concentration} and Lemma~\ref{lemma:eigenvalue_concentration2} for more details) to give upper and lower bounds regarding the matrices $ \Zb\bLambda^2 \Zb^\top $ and $ \Zb\bLambda \Zb^\top $ respectively:
\begin{align*}
    & \big\| \Zb\bLambda \Zb^\top - \tr(\bSigma)\cdot \Ib \big\|_2 \leq c_1\cdot \big( n \| \bSigma\|_2  + \sqrt{n} \| \bSigma\|_F \big), \\
    & \big\| \Zb\bLambda^2 \Zb^\top - \|\bSigma\|_F^2\cdot \Ib \big\|_2 \leq c_1\cdot \big( n \| \bSigma\|_2^2  + \sqrt{n} \| \bSigma^2 \|_F \big),
\end{align*}
where $c_1$ is an absolute constant. Plugging the above inequalities and the bound $\| \ab \|_2^2 = O(D^{-2} n)$ into \eqref{eq:main_lemmaproof_eq2}, we obtain with some calculation that
\begin{align*}
    I_3 \leq c_2 H(\bmu, \Qb, \yb, \bSigma)\cdot (  n\cdot \|\bSigma\|_F^2 +  n^2\cdot \| \bSigma\|_2^2)
\end{align*}
with $ H(\bmu, \Qb, \yb, \bSigma) = [D\cdot \tr(\bSigma) ]^{-2} $, where $c_2$ is an absolute constant. This gives the bound of $I_3$.


Lemma~\ref{lemma:I1I2I3bounds} is significant in three-fold. First of all, the result does not have an explicit dependency on $d$, which makes it applicable to infinite dimensional data. 
Second, Lemma~\ref{lemma:I1I2I3bounds} gives bounds with great simplicity, and shows that the three bounds share a same strictly positive factor $ H(\bmu, \Qb, \yb, \bSigma)$, which can be canceled out since our final goal is to bound the ratio $I_1 / (I_2 + I_3)$. Lastly, Lemma~\ref{lemma:I1I2I3bounds} reveals the fact that the risk bound only depends on $\|\bSigma\|_F$ and $ \| \bSigma\|_2$, which can be small even though the assumption requires $\tr(\bSigma)$ to be large.

We are now ready to present the proof of Theorem~\ref{thm:BOnew}.

\begin{proof}[Proof of Theorem~\ref{thm:BOnew}]
Clearly, under the assumptions of Theorem~\ref{thm:BOnew}, the conditions in Lemma~\ref{lemma:thetaSVM_riskbound} and Lemma~\ref{lemma:I1I2I3bounds} are both satisfied. 
By Lemma~\ref{lemma:I1I2I3bounds}, we have
\begin{align*}
    &\frac{ [ \yb^\top (\Xb\Xb^\top)^{-1} \Xb \bmu  ]^2 }{(\yb^\top (\Xb\Xb^\top)^{-1} \yb)^2 \| \bmu \|_{\bSigma}^2 + \| \Qb^\top (\Xb\Xb^\top)^{-1} \yb \|_{\bSigma}^2} \geq c_1 \cdot \frac{ n^2 \| \bmu \|_2^4 }{  n^2 \| \bmu \|_{\bSigma}^2 + n\cdot \|\bSigma\|_F^2 +  n^2\cdot \| \bSigma\|_2^2 },
\end{align*}
where $c_1$ is an absolute constant. 
Therefore by Lemma~\ref{lemma:thetaSVM_riskbound} we have
\begin{align*}
    R(\hat\btheta_{\text{SVM}}) \leq \exp\Bigg(  \frac{ - c_2 n \| \bmu \|_2^4  }{  n  \| \bmu \|_{\bSigma}^2+ \| \bSigma\|_F^2 +  n\| \bSigma\|_2^2 } \Bigg)
\end{align*}
for some absolute constant $c_2$. Note that by union bound, the above inequality holds with probability at least $ 1 - O(n^{-2}) \geq 1 - n^{-1}$ when $n$ is large enough. This completes the proof. 
\end{proof}


\section{Conclusion and Future Work}
We have studied the benign overfitting phenomenon for sub-Gaussian mxiture models, and established a population risk bound for the maximum margin classifier. Our population risk bound is general and covers both the isotropic and anisotropic settings. When reduced to the isotropic setting, our bound is tighter than existing results. We have also studied a class of non-isotropic models which can be benign even for infinite-dimensional data.

\CC{An interesting future work direction is to study the relation between the dimension and the population risk and verify the double descent phenomenon. Studying  benign overfitting for more complicated learning models such as neural networks is another important future work direction. }

\section*{Acknowledgments and Disclosure of Funding}

We thank the anonymous reviewers for their helpful comments. This work was done when YC was a postdoctoral researcher at UCLA. YC and QG are partially supported by the National Science Foundation award IIS-1903202 and IIS-2008981. MB acknowledges support from NSF IIS-1815697 and the support of the NSF and the Simons Foundation for the Collaboration on the Theoretical Foundations of Deep Learning through awards DMS-2031883 and \#814639. The views and conclusions contained in this paper are those of the authors and should not be interpreted as representing any funding agencies.

\bibliography{deeplearningreference}
\bibliographystyle{ims}

\section*{Checklist}

\begin{enumerate}

\item For all authors...
\begin{enumerate}
  \item Do the main claims made in the abstract and introduction accurately reflect the paper's contributions and scope?
    \answerYes{} 
  \item Did you describe the limitations of your work?
    \answerYes{} 
  \item Did you discuss any potential negative societal impacts of your work?
    \answerNA{The focus of this paper is on theoretical analysis. The results do not have any negative social impact.} 
  \item Have you read the ethics review guidelines and ensured that your paper conforms to them?
    \answerYes{} 
\end{enumerate}

\item If you are including theoretical results...
\begin{enumerate}
  \item Did you state the full set of assumptions of all theoretical results?
    \answerYes{} 
	\item Did you include complete proofs of all theoretical results?
    \answerYes{} 
\end{enumerate}

\item If you ran experiments...
\begin{enumerate}
  \item Did you include the code, data, and instructions needed to reproduce the main experimental results (either in the supplemental material or as a URL)?
    \answerYes{} 
  \item Did you specify all the training details (e.g., data splits, hyperparameters, how they were chosen)?
    \answerYes{} 
	\item Did you report error bars (e.g., with respect to the random seed after running experiments multiple times)?
    \answerNA{The experiments are to verify high-probability guarantees on synthetic data.}
	\item Did you include the total amount of compute and the type of resources used (e.g., type of GPUs, internal cluster, or cloud provider)?
    \answerNA{This paper does not focus on computing cost. All experiments can be run very efficiently on a standard PC.}
\end{enumerate}

\item If you are using existing assets (e.g., code, data, models) or curating/releasing new assets...
\begin{enumerate}
  \item If your work uses existing assets, did you cite the creators?
    \answerNA{}
  \item Did you mention the license of the assets?
    \answerNA{}
  \item Did you include any new assets either in the supplemental material or as a URL?
    \answerNA{}
  \item Did you discuss whether and how consent was obtained from people whose data you're using/curating?
    \answerNA{}
  \item Did you discuss whether the data you are using/curating contains personally identifiable information or offensive content?
    \answerNA{}
\end{enumerate}

\item If you used crowdsourcing or conducted research with human subjects...
\begin{enumerate}
  \item Did you include the full text of instructions given to participants and screenshots, if applicable?
    \answerNA{}
  \item Did you describe any potential participant risks, with links to Institutional Review Board (IRB) approvals, if applicable?
    \answerNA{}
  \item Did you include the estimated hourly wage paid to participants and the total amount spent on participant compensation?
    \answerNA{}
\end{enumerate}

\end{enumerate}

\newpage
\appendix




\section{Completing the Proof of Theorem~\ref{thm:BOnew}}
In Section~\ref{section:proof_main} we give the proof of Theorem~\ref{thm:BOnew} based on several technical propositions and lemmas. Here we present the complete proof of these propositions and lemmas.

\subsection{Proof of Proposition~\ref{prop:interpolationregression}}
\label{section:proof_proposition}
Here we present the proof of Proposition~\ref{prop:interpolationregression}. We begin with the following lemma, which is studied by \citet{muthukumar2020classification,hsu2020proliferation}. 


\begin{lemma}[\citet{hsu2020proliferation}]\label{lemma:equivalence}
$\hat\btheta_{\text{SVM}} = \hat\btheta_{\text{LS}}$ if and only if 
$\yb^\top (\Xb\Xb^\top)^{-1} \eb_i y_i >0$ for all $i\in[n]$. 
\end{lemma}

According to Lemma~\ref{lemma:equivalence}, to study the equivalence between the maximum margin classifier and the minimum norm interpolator, it suffices to derive sufficient conditions such that $\yb^\top (\Xb\Xb^\top)^{-1} \eb_i y_i$, $i\in[n]$ are strictly positive with high probability. We have the following lemma which summarizes some calculations regarding the quantity $\yb^\top (\Xb\Xb^\top)^{-1} \eb_i y_i$. 

\begin{lemma}\label{lemma:condition_calculation}
    Suppose that 
    $\tr(\bSigma) > C\max\{ n^{3/2}  \| \bSigma\|_2 , n \| \bSigma\|_F, n \| \bmu \|_{\bSigma}\}$ 
    for some absolute constant $C$. Then with probability at least $ 1 - O(n^{-2}) $, 
    \begin{align*}
        \yb^\top (\Xb\Xb^\top)^{-1} \eb_i y_i\geq G \Big[ 1 - C' n\big|\bmu^\top\Qb^\top (\Qb\Qb^\top)^{-1}\eb_i \big|   \Big]
    \end{align*}
    for all $i\in [n]$, 
    where $G = G(\bmu, \Qb, \yb, \bSigma)>0 $ is a strictly positive factor and $C'>0$ is an absolute constant. 
\end{lemma}
By Lemma~\ref{lemma:condition_calculation}, we can see that in order to ensure $\yb^\top (\Xb\Xb^\top)^{-1} \eb_i y_i > 0$, it suffices to establish an upper bound for
$|\bmu^\top\Qb^\top (\Qb\Qb^\top)^{-1}\eb_i | $. 
However, deriving tight upper bounds for this term turns out to be challenging, as a simple application of the Cauchy-Schwarz inequality can lead to a loose bound with an additional $\sqrt{n}$ factor. In the following, we establish a refined bound on the term $|\bmu^\top\Qb^\top (\Qb\Qb^\top)^{-1}\eb_i | $.

\begin{lemma}\label{lemma:anisotropicbound1}
Suppose that 
    $\tr(\bSigma) > C\max\{ n^{3/2}  \| \bSigma\|_2 , n \| \bSigma\|_F \}$  
for some absolute constant $C$. Then with probability at least $ 1 - O(n^{-2}) $, 
\begin{align*}
    \big| \bmu^\top\Qb^\top (\Qb\Qb^\top)^{-1} \eb_i \big| \leq   \frac{C' \| \bmu \|_{\bSigma}  \cdot  \sqrt{\log(n)}}{ \tr(\bSigma)  }
\end{align*}
for all $i\in [n]$, where $C'>0$ is an absolute constant.
\end{lemma}

We are now ready to present the proof of Proposition~\ref{prop:interpolationregression}

\begin{proof}[Proof of Proposition~\ref{prop:interpolationregression}]
By the union bound, we have that with probability at least $1 - 2n^{-2}$, the results in Lemma~\ref{lemma:condition_calculation} and Lemma~\ref{lemma:anisotropicbound1} both hold. Therefore, for any $i\in[n]$, we have
\begin{align*}
    \yb^\top (\Xb\Xb^\top)^{-1} \eb_i y_i&\geq G \Big[ 1 - c_1 n\big|\bmu^\top\Qb^\top (\Qb\Qb^\top)^{-1}\eb_i \big|   \Big] \geq G \Bigg[ 1 - \frac{ c_2 n  \sqrt{\log(n)} \cdot \| \bmu \|_{\bSigma} }{ \tr(\bSigma)  } \Bigg]\\
    &\propto  \tr(\bSigma) - c_2 n \sqrt{\log(n)}\cdot \| \bmu \|_{\bSigma}.
\end{align*}
By the assumption $\tr( \bSigma ) \geq C n\sqrt{\log(n)}\cdot \| \bmu \|_{\bSigma} $ for some large enough absolute constant $C$, we have $\yb^\top (\Xb\Xb^\top)^{-1} \eb_i y_i > 0$. Finally, applying Lemma~\ref{lemma:equivalence}, we conclude that $\hat\btheta_{\text{SVM}} = \hat\btheta_{\text{LS}}$.
\end{proof}

\subsection{Proof of Lemmas in Section~\ref{section:proof_main}}\label{section:proof_appendix1}
We denote $\bnu = \Qb \bmu$ and $\Ab = \Qb\Qb^\top$. Based on these notations, in the following we present several basic lemmas that are used in our proof. We have the following lemma which gives concentration inequalities for the the eigenvalues of $\Ab$.


\begin{lemma}\label{lemma:eigenvalue_concentration}
    With probability at least $1 - n^{-2}$,
    \begin{align*}
    \big\| \Ab  - \tr(\bSigma)\cdot \Ib \big\|_2 \leq \epsilon_{\lambda} := C \sigma_u^2  \big( n\cdot \| \bSigma\|_2  + \sqrt{n}\cdot \| \bSigma\|_F \big),
\end{align*}
where $C$ is an absolute constant.
\end{lemma}

The following lemma presents some calculations on the quantity $\yb^\top (\Xb\Xb^\top)^{-1}$. It utilizes a result introduced in \citet{wang2020benign}, which is based on the application of the Sherman–Morrison–Woodbury formula.


\begin{lemma}\label{lemma:matrixcalculation}
The following calculation of $\yb^\top (\Xb\Xb^\top)^{-1}$ holds: 
\begin{align*}
    \yb^\top (\Xb\Xb^\top)^{-1} = D^{-1} [ ( 1 + \yb^\top \Ab^{-1} \bnu )\cdot \yb^\top \Ab^{-1} - \yb^\top \Ab^{-1} \yb \cdot \bnu^\top \Ab^{-1}],
\end{align*}
where 
$D = \yb^\top\Ab^{-1} \yb \cdot  (\| \bmu\|_2^2 - \bnu^\top \Ab^{-1} \bnu)  + (1 + \yb^\top\Ab^{-1} \bnu)^2> 0$.
\end{lemma}

Motivated by Lemma~\ref{lemma:matrixcalculation}, we estimate the orders of the terms $ \yb^\top \Ab^{-1} \yb$, $\bnu^\top \Ab^{-1} \bnu$, and $\yb^\top \Ab^{-1} \bnu$. The results are given in the following lemma.

\begin{lemma}\label{lemma:concentrationbounds}
Let $\epsilon_{\lambda}$ be defined in Lemma~\ref{lemma:eigenvalue_concentration}, and suppose that $\tr( \bSigma )  > \epsilon_{\lambda}$. Then with probability at least $1 - O(n^{-2})$, the following inequalities hold: 
\begin{align*}
    &\frac{n}{ \tr( \bSigma )  + \epsilon_{\lambda} } \leq  \yb^\top \Ab^{-1} \yb \leq \frac{n}{ \tr( \bSigma )  - \epsilon_{\lambda} },\\
    &\frac{n  - C \sqrt{n\log( n)}}{ \tr( \bSigma )  + \epsilon_{\lambda} } \cdot \| \bmu \|_{\bSigma}^2 \leq  \bnu^\top \Ab^{-1} \bnu \leq \frac{ n  + C \sqrt{n\log( n)} }{ \tr( \bSigma )  - \epsilon_{\lambda} } \cdot \| \bmu \|_{\bSigma}^2,\\
    &|\yb^\top \Ab^{-1} \bnu| \leq \frac{C n   }{  \tr( \bSigma )  - \epsilon_{\lambda}  } \| \bmu \|_{\bSigma},
\end{align*}
where $C$ is an absolute constant.


\end{lemma}


\subsubsection{Proof of Lemma~\ref{lemma:subGaussian_riskbound}}\label{section:proof_subGaussian_riskbound}
Here we give the detailed proof of Lemma~\ref{lemma:subGaussian_riskbound}, which is based on the one-side sub-Gaussian tail bound.
\begin{proof}[Proof of Lemma~\ref{lemma:subGaussian_riskbound}]
By definition, we have
\begin{align*}
    R(\btheta) &= \PP( y\cdot \btheta^\top \xb < 0 )  = \PP[ y\cdot \btheta^\top ( y\cdot \bmu + \qb) < 0 ] = \PP[ \btheta^\top \bmu  < y\cdot \btheta^\top \qb ] = \PP[ \btheta^\top \bmu  < y\cdot \btheta^\top  \Vb \bLambda^{1/2} \ub ],
\end{align*}
where in the second and last equations we plug in the definitions of $\xb$ and $\qb$ according to our data generation procedure described in Section~\ref{section:problemsetting}. Note that $\ub$ has independent, $\sigma_u$-sub-Gaussian entries. Therefore we have

\begin{align*}
    \|  \btheta^\top  \Vb \bLambda^{1/2} \ub  \|_{\psi_2} \leq c_1  \|  \btheta^\top  \Vb \bLambda^{1/2} \|_{2} = c_1  \sqrt{ \btheta^\top  \Vb \bLambda \Vb^\top \btheta } = c_1  \sqrt{ \btheta^\top  \bSigma \btheta }.
\end{align*}
Applying the one-side sub-Gaussian tail bound (e.g., Theorem~A.2 in \citet{chatterji2020finite}) completes the proof. 
\end{proof}

\subsubsection{Proof of Lemma~\ref{lemma:thetaSVM_riskbound}}
The proof of Lemma~\ref{lemma:thetaSVM_riskbound} is given as follows, where we utilize Proposition~\ref{prop:interpolationregression} and Lemma~\ref{lemma:subGaussian_riskbound} to derive the desired bound.

\begin{proof}[Proof of Lemma~\ref{lemma:thetaSVM_riskbound}]
By Proposition~\ref{prop:interpolationregression}, we have
\begin{align*}
    \hat\btheta_{\text{SVM}} = \hat\btheta_{\text{LS}}  = \Xb^\top (\Xb\Xb^\top)^{-1} \yb.
\end{align*}
Plugging it into the risk bound in Lemma~\ref{lemma:subGaussian_riskbound}, we obtain
\begin{align*}
    R(\hat\btheta_{\text{SVM}} ) \leq \exp\bigg\{ - \frac{ C[ \yb^\top (\Xb\Xb^\top)^{-1} \Xb \bmu  ]^2}{ \| \Xb^\top (\Xb\Xb^\top)^{-1} \yb \|_{\bSigma}^2 } \bigg\}.
\end{align*}
Note that based on our model, we have $\Xb = \yb \bmu^\top + \Qb$, and 
\begin{align}
    \| \Xb^\top (\Xb\Xb^\top)^{-1} \yb \|_{\bSigma}^2 
    &\quad= \| (\yb \bmu^\top + \Qb)^\top (\Xb\Xb^\top)^{-1} \yb \|_{\bSigma}^2  \nonumber\\
    &\quad\leq 2 \| \bmu \yb^\top (\Xb\Xb^\top)^{-1} \yb \|_{\bSigma}^2 + 2\| \Qb^\top (\Xb\Xb^\top)^{-1} \yb \|_{\bSigma}^2 \nonumber\\
    &\quad= 2 (\yb^\top (\Xb\Xb^\top)^{-1} \yb)^2 \cdot \| \bmu \|_{\bSigma}^2 + 2\| \Qb^\top (\Xb\Xb^\top)^{-1} \yb \|_{\bSigma}^2.\nonumber
\end{align}
Therefore we have
\begin{align*}
    R(\hat\btheta_{\text{SVM}} ) \leq \exp\bigg\{ \frac{ - (C/2)\cdot [ \yb^\top (\Xb\Xb^\top)^{-1} \Xb \bmu  ]^2 }{(\yb^\top (\Xb\Xb^\top)^{-1} \yb)^2 \cdot \| \bmu \|_{\bSigma}^2 + \| \Qb^\top (\Xb\Xb^\top)^{-1} \yb \|_{\bSigma}^2 } \bigg\}.
\end{align*}
This completes the proof.
\end{proof}

\subsubsection{Proof of Lemma~\ref{lemma:I1I2I3bounds}}
In this subsection we present the proof of Lemma~\ref{lemma:I1I2I3bounds}. We first give the following lemma, which follows by exactly the same proof as Lemma~\ref{lemma:eigenvalue_concentration}.

\begin{lemma}\label{lemma:eigenvalue_concentration2}
    Suppose that $\Zb \in \RR^{n\times d}$ is a random matrix with i.i.d. sub-Gaussian entries with sub-Gaussian norm $\sigma_u$. Then with probability at least $1 - O(n^{-2})$,
    \begin{align*}
    \big\| \Zb\bLambda^2 \Zb^\top   - \|\bSigma\|_F^2\cdot \Ib \big\|_2 \leq \epsilon_{\lambda}' := C \sigma_u^2  \big( n\cdot \| \bSigma\|_2^2  + \sqrt{n}\cdot \| \bSigma^2\|_F \big),
\end{align*}
where $C$ is an absolute constant.
\end{lemma}

Based on Lemma~\ref{lemma:eigenvalue_concentration2}, we can give the proof of Lemma~\ref{lemma:I1I2I3bounds} as follows.

\begin{proof}[Proof of Lemma~\ref{lemma:I1I2I3bounds}]
We first derive the lower bound for $I_1$. By Lemma~\ref{lemma:matrixcalculation} and the model definition $\Xb = \yb \bmu^\top + \Qb$, we have
\begin{align}
    \yb^\top (\Xb\Xb^\top)^{-1}  \Xb \bmu &= D^{-1} [ ( 1 + \yb^\top \Ab^{-1} \bnu ) \yb^\top \Ab^{-1} - \yb^\top \Ab^{-1} \yb \cdot \bnu^\top \Ab^{-1}] (\yb \bmu^\top + \Qb)\bmu \nonumber\\
    &= D^{-1} [ ( 1 + \yb^\top \Ab^{-1} \bnu ) \yb^\top \Ab^{-1} - \yb^\top \Ab^{-1} \yb \cdot \bnu^\top \Ab^{-1}] (\yb \cdot \| \bmu \|_2^2 + \Qb\bmu) \nonumber\\
    & = D^{-1} [ ( 1 + \yb^\top \Ab^{-1} \bnu ) \yb^\top \Ab^{-1} \yb - \yb^\top \Ab^{-1} \yb \cdot \bnu^\top \Ab^{-1} \yb] \cdot \| \bmu \|_2^2 \nonumber \\
    & \qquad + D^{-1} [ ( 1 + \yb^\top \Ab^{-1} \bnu ) \yb^\top \Ab^{-1}\bnu - \yb^\top \Ab^{-1} \yb \cdot \bnu^\top \Ab^{-1}\bnu] )\nonumber\\
    & = D^{-1}\cdot [ (\| \bmu \|_2^2 - \bnu^\top \Ab^{-1}\bnu) \yb^\top \Ab^{-1} \yb  + ( 1 + \yb^\top \Ab^{-1} \bnu ) \yb^\top \Ab^{-1}\bnu  ],\label{eq:refinedBO_numerator_eq0}
\end{align}
where the third equality follows by the notation $\bnu = \Qb\bmu$.
By Lemma~\ref{lemma:concentrationbounds} and the assumption that $\tr( \bSigma ) \geq C \max\{ \epsilon_{\lambda} , n \|\bSigma \|_2 , n \| \bmu \|_{\bSigma} \}$ for some large enough constant $C$, when $n$ is large enough we have
\begin{align*}
     &| \yb^\top \Ab^{-1} \bnu | \leq \frac{ c_1 n   }{  \tr( \bSigma )  - \epsilon_{\lambda}  } \| \bmu \|_{\bSigma} \leq \frac{ 2c_1 n   }{  \tr( \bSigma ) } \| \bmu \|_{\bSigma} \leq 1,\\
     &0\leq \bnu^\top \Ab^{-1}\bnu \leq \frac{ n  + c_2 \sqrt{n\log( n)} }{ \tr( \bSigma )  - \epsilon_{\lambda} } \cdot \| \bmu \|_{\bSigma}^2 \leq \frac{ 2n }{ \tr( \bSigma ) } \cdot \| \bmu \|_{\bSigma}^2 \leq \frac{ 2n \| \bSigma \|_2 }{ \tr( \bSigma ) } \cdot \| \bmu \|_2^2 \leq \frac{1}{2}\cdot  \| \bmu \|_2^2 ,\\
     & \yb^\top \Ab^{-1} \yb \geq \frac{n}{ \tr( \bSigma )  + \epsilon_{\lambda} } \geq \frac{n}{ 2\tr( \bSigma ) },
\end{align*}
where $c_1,c_2$ are absolute constants. 
Plugging the bounds above into \eqref{eq:refinedBO_numerator_eq0}, we obtain
\begin{align*}
    | \yb^\top (\Xb\Xb^\top)^{-1}  \Xb \bmu | &\geq D^{-1}\cdot  \bigg(\frac{1}{2}\cdot\| \bmu \|_2^2\cdot \yb^\top \Ab^{-1} \yb  - 2\cdot   |\yb^\top \Ab^{-1}\bnu | \bigg)\\
    &\geq D^{-1}\cdot \bigg[ \frac{n}{ 4\tr( \bSigma ) } \cdot \| \bmu \|_2^2 - \frac{ 4n   }{  \tr( \bSigma ) } \| \bmu \|_{\bSigma} \bigg]\\
    &\geq D^{-1}\cdot \frac{n}{ 4\tr( \bSigma ) } \cdot  ( \| \bmu \|_2^2 -16 \| \bmu \|_{\bSigma} )\\
    &\geq D^{-1}\cdot \frac{n}{ 8\tr( \bSigma ) } \cdot  \| \bmu \|_2^2,
\end{align*}
where the last inequality follows by the assumption that $\| \bmu \|_2^2 \geq C \| \bmu\|_{\bSigma} $ for some large enough absolute constant $C$. 
Therefore we have
\begin{align*}
    [\yb^\top (\Xb\Xb^\top)^{-1}  \Xb \bmu]^2 \geq D^{-2}\cdot \frac{n^2}{ 64 [\tr( \bSigma )]^2 } \cdot   \| \bmu \|_2^4 = \frac{H(\bmu, \Qb, \yb, \bSigma) }{64} \cdot n^2 \| \bmu \|_2^4,
\end{align*}
where we define 
$$ H(\bmu, \Qb, \yb, \bSigma) := [D\cdot \tr(\bSigma) ]^{-2} > 0. $$
This completes the proof of the lower bound of $I_1$. 



For $I_2$, by Lemma~\ref{lemma:matrixcalculation} we have 
\begin{align*}
    \yb^\top (\Xb\Xb^\top)^{-1} \yb &=  D^{-1} [ ( 1 + \yb^\top \Ab^{-1} \bnu ) \yb^\top \Ab^{-1} \yb- \yb^\top \Ab^{-1} \yb \cdot \bnu^\top \Ab^{-1} \yb] \\
    & = D^{-1} [ ( 1 + \yb^\top \Ab^{-1} \bnu ) \yb^\top \Ab^{-1} \yb- \yb^\top \Ab^{-1} \yb \cdot \bnu^\top \Ab^{-1} \yb] \\
    &= D^{-1} \cdot \yb^\top \Ab^{-1} \yb\\
    &\leq D^{-1}  \cdot \frac{n}{ \tr( \bSigma )  - \epsilon_{\lambda} }\\
    &\leq  2D^{-1}  \cdot \frac{n}{ \tr( \bSigma )},
\end{align*}
where the first inequality follows by Lemma~\ref{lemma:concentrationbounds}, and the second inequality follows by the assumption that $\tr( \bSigma ) \geq C \epsilon_{\lambda}$ for some large enough constant $C$. Therefore we have
\begin{align*}
    I_2 =  (\yb^\top (\Xb\Xb^\top)^{-1} \yb)^2 \cdot  \| \bmu \|_{\bSigma}^2 \leq  4D^{-2}  \cdot \frac{n^2 \cdot \| \bmu \|_{\bSigma}^2}{ [\tr( \bSigma )]^2} = 4 H(\bmu, \Qb, \yb, \bSigma) \cdot n^2 \cdot \| \bmu \|_{\bSigma}^2, 
\end{align*}
where we use the definition $H(\bmu, \Qb, \yb, \bSigma) =  [D\cdot \tr(\bSigma) ]^{-2}$.
This proves the upper bound of $I_2$.

For $I_3$, by our calculation in Lemma~\ref{lemma:matrixcalculation}, we have
\begin{align*}
    \yb^\top (\Xb\Xb^\top)^{-1} = D^{-1} [ ( 1 + \yb^\top \Ab^{-1} \bnu ) \yb - \yb^\top \Ab^{-1} \yb \cdot \bnu]^\top \Ab^{-1}.
\end{align*}
Denote $\ab = D^{-1}[( 1 + \yb^\top \Ab^{-1} \bnu )\cdot \yb - \yb^\top \Ab^{-1} \yb \cdot \bnu]$. Then
\begin{align}
    I_3 &= \yb^\top (\Xb\Xb^\top)^{-1} \Qb \bSigma \Qb^\top (\Xb\Xb^\top)^{-1} \yb \nonumber \\
    &= \ab^\top (\Qb\Qb^\top)^{-1}\Qb\bSigma\Qb^\top (\Qb\Qb^\top)^{-1} \ab \nonumber \\
    & = \ab^\top (\Zb\bLambda \Zb^\top)^{-1} \Zb\bLambda^2 \Zb^\top (\Zb\bLambda \Zb^\top)^{-1} \ab, \label{eq:refinedBOderivation_eq1}
\end{align}
where we plug in $\bSigma = \Vb \bLambda \Vb^\top$ and $\Qb = \Zb \bLambda^{1/2} \Vb^\top$ for $\Zb$ with independent sub-Gaussian entries. By Lemma~\ref{lemma:eigenvalue_concentration}, Lemma~\ref{lemma:eigenvalue_concentration2} and \eqref{eq:refinedBOderivation_eq1}, when $\tr(\bSigma) \geq \epsilon_{\lambda}$ we have
\begin{align}
    I_3&= \ab^\top (\Zb\bLambda \Zb^\top)^{-1} \Zb\bLambda^2 \Zb^\top (\Zb\bLambda \Zb^\top)^{-1} \ab \nonumber\\
    & \leq \ab^\top (\Zb\bLambda \Zb^\top)^{-2} \ab \cdot \big[\|\bSigma\|_F^2 + \epsilon_{\lambda}'\big] \nonumber\\
    & \leq \| \ab \|_2^2\cdot \frac{\|\bSigma\|_F^2 + \epsilon_{\lambda}'}{ [\tr(\bSigma) - \epsilon_{\lambda}]^2}.\label{eq:refinedBOderivation_eq1.5}
\end{align}
Here the first inequality follows by Lemma~\ref{lemma:eigenvalue_concentration2}, and the second inequality follows by Lemma~\ref{lemma:eigenvalue_concentration}. 
By definition, we have
\begin{align}
    \| \ab \|_2^2 & = \| D^{-1}( 1 + \yb^\top \Ab^{-1} \bnu ) \yb - \yb^\top \Ab^{-1} \yb \cdot \bnu \|_2^2 \nonumber \\
    & \leq 2 D^{-2} ( 1 + \yb^\top \Ab^{-1} \bnu )^2 \| \yb \|_2^2 + 2  D^{-2} ( \yb^\top \Ab^{-1} \yb  )^2\cdot  \|  \Qb \bmu\|_2^2. \nonumber
\end{align}

Then with the same proof as in Lemma~\ref{lemma:concentrationbounds}, when $n$ is sufficiently large, with probability at least $1- O(n^{-2})$ we have 
\begin{align*}
      \| \Qb \bmu \|_2^2 \leq  2n \| \bmu \|_{\bSigma}^2.
\end{align*}
Therefore we have
\begin{align}
    \| \ab \|_2^2 & \leq 2 D^{-2} ( 1 + \yb^\top \Ab^{-1} \bnu )^2 \| \yb \|_2^2 + 2  D^{-2} ( \yb^\top \Ab^{-1} \yb  )^2\cdot  \|  \Qb \bmu\|_2^2 \nonumber\\
    & \leq 2 D^{-2} ( 1 + \yb^\top \Ab^{-1} \bnu )^2 \cdot n + 4  D^{-2} ( \yb^\top \Ab^{-1} \yb  )^2\cdot  n \cdot \| \bmu \|_{\bSigma}^2. \label{eq:refinedBOderivation_eq2}
\end{align}

Moreover, by Lemma~\ref{lemma:concentrationbounds} and the assumption that $\tr( \bSigma ) \geq C \max\{ \epsilon_{\lambda} , n, n \| \bmu \|_{\bSigma} \}$ for some large enough constant $C$, we have
\begin{align*}
    &|\yb^\top \Ab^{-1} \bnu| \leq \frac{c_3 n   }{  \tr( \bSigma )  - \epsilon_{\lambda}  } \| \bmu \|_{\bSigma} \leq \sqrt{2} - 1, \\
    &  \yb^\top \Ab^{-1} \yb \leq \frac{n}{ \tr( \bSigma )  - \epsilon_{\lambda} } \leq \frac{2n}{\tr( \bSigma )},
\end{align*}
where $c_3$ is an absolute constant. Plugging the above bounds into \eqref{eq:refinedBOderivation_eq2}, we obtain
\begin{align*}
    \| \ab \|_2^2 & \leq 2 D^{-2} ( 1 + \yb^\top \Ab^{-1} \bnu )^2 \cdot n + 4  D^{-2} ( \yb^\top \Ab^{-1} \yb  )^2\cdot  n \cdot \| \bmu \|_{\bSigma}^2\\
    &\leq 4 D^{-2} \cdot n +  8 D^{-2} \cdot n  \cdot \bigg[\frac{n}{\tr( \bSigma )} \cdot \| \bmu \|_{\bSigma}\bigg]^2\\
    &\leq 5 D^{-2} \cdot n,
\end{align*}
where the last inequality utilizes the assumption  $\tr( \bSigma ) \geq C n \| \bmu \|_{\bSigma} $ for some large enough constant $C$ again. Further plugging this bound into \eqref{eq:refinedBOderivation_eq1.5}, we obtain
\begin{align}
    I_3 &\leq \| \ab \|_2^2\cdot \frac{\|\bSigma\|_F^2 + \epsilon_{\lambda}'}{ [\tr(\bSigma) - \epsilon_{\lambda}]^2} \leq 5 D^{-2} n \cdot \frac{\|\bSigma\|_F^2 + \epsilon_{\lambda}'}{ [\tr(\bSigma) - \epsilon_{\lambda}]^2} \nonumber\\
    &\leq c_4 D^{-2}\cdot  \frac{ n\cdot \|\bSigma\|_F^2 +  n^2\cdot \| \bSigma\|_2^2  + n^{3/2}\cdot \| \bSigma^2\|_F}{ [\tr(\bSigma) ]^2},\label{eq:refinedBOderivation_I3bound_eq1}
\end{align}
where $c_4$ is an absolute constant. 
Note that we have
\begin{align*}
     n^{3/2}\cdot \| \bSigma^2\|_F \leq n\cdot \| \bSigma\|_F\cdot (\sqrt{n}\cdot \| \bSigma\|_2)\leq n\cdot ( \| \bSigma\|_F^2 + n\cdot \| \bSigma\|_2^2)/ 2.
\end{align*}
Plugging this bound into \eqref{eq:refinedBOderivation_I3bound_eq1}, we have
\begin{align*}
    I_3 \leq c_5 D^{-2}\cdot\frac{ n\cdot \|\bSigma\|_F^2 +  n^2\cdot \| \bSigma\|_2^2 }{ [\tr(\bSigma) ]^2} = c_5 H(\bmu, \Qb, \yb, \bSigma)\cdot ( n\cdot \|\bSigma\|_F^2 +  n^2\cdot \| \bSigma\|_2^2 ),
\end{align*}
where we use the definition $H(\bmu, \Qb, \yb, \bSigma) =  [D\cdot \tr(\bSigma) ]^{-2}$, and $c_4$ is an absolute constant. This finishes the proof of the upper bound of $I_3$. 
\end{proof}

\subsection{Proof of Lemmas in Appendix~\ref{section:proof_proposition}}
We present the proofs of Lemmas~\ref{lemma:condition_calculation} and \ref{lemma:anisotropicbound1}.

\subsubsection{Proof of Lemma~\ref{lemma:condition_calculation}}

Here we present the proof of Lemma~\ref{lemma:condition_calculation}. The proof utilizes Lemma~\ref{lemma:matrixcalculation} and an argument based on the polarization identity.



\begin{proof}[Proof of Lemma~\ref{lemma:condition_calculation}] 
By Lemma~\ref{lemma:matrixcalculation}, we have
\begin{align}\label{eq:proof_anisotropicbound2_eq0}
    \yb^\top (\Xb\Xb^\top)^{-1} \eb_i y_i = D^{-1} [ ( 1 + \yb^\top \Ab^{-1} \bnu ) \yb^\top \Ab^{-1} \eb_i y_i - \yb^\top \Ab^{-1} \yb \cdot \bnu^\top \Ab^{-1} \eb_i y_i].
\end{align}
Moreover,  by definition we have
\begin{align}
    \yb^\top \Ab^{-1} \eb_i y_i &= \frac{1}{4\sqrt{n}} (\yb + \sqrt{n}\eb_i y_i )^\top \Ab^{-1} (\yb + \sqrt{n}\eb_i y_i ) - \frac{1}{4\sqrt{n}} (\yb - \sqrt{n}\eb_i y_i )^\top \Ab^{-1} (\yb - \sqrt{n}\eb_i y_i )\nonumber \\
    &\geq \frac{1}{4\sqrt{n}} \bigg[ \frac{\| \yb + \sqrt{n}\eb_i y_i \|_2^2}{\tr(\bSigma)  + \epsilon_{\lambda}} - \frac{\| \yb - \sqrt{n}\eb_i y_i \|_2^2}{\tr(\bSigma)  - \epsilon_{\lambda}}  \bigg]\nonumber \\
    &= \frac{1}{4\sqrt{n}} \bigg[ \frac{2n + 2\sqrt{n}}{\tr(\bSigma)  + \epsilon_{\lambda}} - \frac{2n - 2\sqrt{n}}{\tr(\bSigma)  - \epsilon_{\lambda}}  \bigg]\nonumber \\
    & = \frac{1}{2\sqrt{n}} \cdot  \frac{(n + \sqrt{n}) (\tr(\bSigma)  - \epsilon_{\lambda}) - (n - \sqrt{n}) (\tr(\bSigma)  + \epsilon_{\lambda}) }{\tr(\bSigma)^2  - \epsilon_{\lambda}^2}\nonumber \\
    & = \frac{1}{2\sqrt{n}} \cdot  \frac{2\sqrt{n} \tr(\bSigma) - 2n \epsilon_{\lambda} }{\tr(\bSigma)^2  - \epsilon_{\lambda}^2}\nonumber \\
     & = \frac{ \tr(\bSigma) - \sqrt{n} \epsilon_{\lambda} }{\tr(\bSigma)^2  - \epsilon_{\lambda}^2}, \label{eq:proof_anisotropicbound2_eq1}
\end{align}
where we use the polarization identity $\ab^\top \Mb \bbb = 1/4(\ab+\bbb)^\top \Mb (\ab+\bbb) - 1/4(\ab-\bbb)^\top \Mb (\ab-\bbb)$ in the first equality and use  Lemma~\ref{lemma:eigenvalue_concentration} 
to derive the inequality.

Plugging \eqref{eq:proof_anisotropicbound2_eq1} and  the inequalities in Lemmas~\ref{lemma:concentrationbounds} into \eqref{eq:proof_anisotropicbound2_eq0}, we have that as long as $\tr(\bSigma) > c_1\max\{ n \| \bmu \|_{\bSigma}, \epsilon_{\lambda}\}$ for some large enough constant $c_1$, $\yb^\top \Ab^{-1} \bnu \leq 1/2$ and therefore
\begin{align}
    \yb^\top (\Xb\Xb^\top)^{-1} \eb_i y_i &= D^{-1} [ ( 1 + \yb^\top \Ab^{-1} \bnu ) \yb^\top \Ab^{-1} \eb_i y_i  - \yb^\top \Ab^{-1} \yb \cdot \bnu^\top \Ab^{-1}\eb_i y_i ] \nonumber \\
    &\geq D^{-1}\cdot \bigg[ \frac{1}{2}\cdot \yb^\top \Ab^{-1} \eb_i y_i  - \frac{ c_2 n }{ \tr(\bSigma)}\cdot |\bnu^\top \Ab^{-1}\eb_i y_i| \bigg],\label{eq:equivalenceproof_eq0}
\end{align}
where $c_2$ is an absolute constant. 
By \eqref{eq:proof_anisotropicbound2_eq1}, we can see that as long as $ \tr(\bSigma) \geq c_3\sqrt{n} \epsilon_{\lambda}$ for some large enough absolute constant $c_3$, we have
\begin{align*}
    \yb^\top \Ab^{-1} \eb_i y_i \geq \frac{ \tr(\bSigma) - \sqrt{n} \epsilon_{\lambda} }{\tr(\bSigma)^2  - \epsilon_{\lambda}^2} \geq \frac{1}{2\tr(\bSigma)}.
\end{align*}
Plugging the bound above into \eqref{eq:equivalenceproof_eq0}, we obtain
\begin{align*}
    \yb^\top (\Xb\Xb^\top)^{-1} \eb_i y_i \geq \frac{1}{4D \tr(\bSigma)}\cdot [ 1 - c_4 n \cdot |\bnu^\top \Ab^{-1}\eb_i y_i| ].
\end{align*}
Since $D > 0$, we see that $G(\bmu, \Qb, \yb, \bSigma):= [4D \tr(\bSigma)]^{-1} > 0$. This completes the proof. 
\end{proof}

\subsubsection{Proof of Lemma~\ref{lemma:anisotropicbound1}}
Here we give the detailed proof of Lemma~\ref{lemma:anisotropicbound1} to backup the proof sketch presented in Section~\ref{section:equivalence}. The proof is based on the polarization identity.

\begin{proof}[Proof of Lemma~\ref{lemma:anisotropicbound1}] We have the following calculation,
\begin{align}
    \bmu^\top\Qb^\top\Ab^{-1} \eb_i y_i &= \frac{1}{\| \Qb\bmu \|_2}\cdot (\Qb\bmu)^\top \Ab^{-1} (\| \Qb\bmu \|_2\cdot \eb_i y_i) \nonumber \\
    & = \frac{1}{4\| \Qb\bmu \|_2}\cdot (\Qb\bmu + \| \Qb\bmu \|_2\cdot \eb_i y_i)^\top \Ab^{-1} (\Qb\bmu + \| \Qb\bmu \|_2\cdot \eb_i y_i) \nonumber \\
    &\quad - \frac{1}{4\| \Qb\bmu \|_2}\cdot (\Qb\bmu - \| \Qb\bmu \|_2\cdot \eb_i y_i)^\top \Ab^{-1} (\Qb\bmu - \| \Qb\bmu \|_2\cdot \eb_i y_i) \nonumber \\
    & \leq \frac{1}{4\| \Qb\bmu \|_2}\cdot \bigg[  \frac{\| \Qb\bmu + \| \Qb\bmu \|_2\cdot \eb_i y_i \|_2^2}{ \tr(\bSigma)- \epsilon_{\lambda}  } - \frac{\| \Qb\bmu - \| \Qb\bmu \|_2\cdot \eb_i y_i \|_2^2}{ \tr(\bSigma)+ \epsilon_{\lambda}  }     \bigg] \nonumber\\
    &= \frac{1}{4\| \Qb\bmu \|_2}\cdot \bigg[  \frac{2 \| \Qb\bmu \|_2^2 + 2 y_i\| \Qb\bmu \|_2\cdot \eb_i^\top \Qb\bmu }{ \tr(\bSigma)- \epsilon_{\lambda}  } - \frac{2\| \Qb\bmu \|_2^2  - 2 y_i\| \Qb\bmu \|_2\cdot \eb_i^\top \Qb\bmu }{ \tr(\bSigma)+ \epsilon_{\lambda}  } \bigg] \nonumber \\
    &= \frac{1}{2\| \Qb\bmu \|_2}\cdot \frac{2\| \Qb\bmu \|_2^2 \cdot \epsilon_{\lambda} + 2 y_i\| \Qb\bmu \|_2\cdot \eb_i^\top \Qb\bmu \cdot \tr(\bSigma)}{ \tr(\bSigma)^2 - \epsilon_{\lambda}^2  }\nonumber\\
    &= \frac{\| \Qb\bmu \|_2 \cdot \epsilon_{\lambda} +  y_i \eb_i^\top \Qb\bmu \cdot \tr(\bSigma)}{ \tr(\bSigma)^2 - \epsilon_{\lambda}^2  },\label{eq:fgbounds_eq1}
\end{align}
where the first equality holds due to the polarization identity $\ab^\top \Mb \bbb = 1/4(\ab+\bbb)^\top \Mb (\ab+\bbb) - 1/4(\ab-\bbb)^\top \Mb (\ab-\bbb)$, and the first inequality follows by Lemma~\ref{lemma:eigenvalue_concentration}. Based on our model assumption, we can denote $\Qb = \Zb \bLambda^{1/2} \Vb^\top$, where the entries of $\Zb$ are independent sub-Gaussian random variables with $ \| \Zb_{ij} \|_{\psi_2} \leq \sigma_u$ for all $i\in [n]$ and $j\in[p]$. Denote $\tilde\bmu = \Lambda^{1/2} \Vb^\top \bmu$. Then with the same proof as in Lemma~\ref{lemma:concentrationbounds}, we have 
\begin{align*}
      \| \Qb \bmu \|_2^2= \| \Zb \tilde\bmu \|_2^2 \leq  2n \| \tilde\bmu \|_2^2  = 2n \| \bmu \|_{\bSigma}^2
\end{align*}
when $n$ is large enough. Moreover, we also have
\begin{align*}
    \| y_i \eb_i^\top \Qb \bmu \|_{\psi_2} = \Bigg\| \sum_{j=1}^p \Zb_{ij} \tilde\mu_j \Bigg\|_{\psi_2} \leq \| \tilde\bmu \|_2 \cdot \sigma_u. 
\end{align*}
Therefore by Hoeffding's inequality, with probability at least $1 - n^{-1}$, we have
\begin{align*}
    | y_i \eb_i^\top \Qb \bmu | \leq c_1 \| \tilde\bmu \|_2 \cdot \sqrt{\log(n)} = c_1 \| \bmu \|_{\bSigma} \cdot \sqrt{\log(n)},
\end{align*}
where $c_1$ is an absolute constant. Therefore we have
\begin{align*}
    \bnu^\top \Ab^{-1} \eb_i y_i &\leq \frac{\sqrt{2n} \| \bmu \|_{\bSigma} \cdot \epsilon_{\lambda} +  c_2 \| \bmu \|_{\bSigma}\sqrt{\log(n)} \cdot \tr(\bSigma)}{ \tr(\bSigma)^2 - \epsilon_{\lambda}^2  }.
\end{align*}
With the exact same proof, we also have
\begin{align*}
    -\bnu^\top \Ab^{-1} \eb_i y_i &\leq \frac{\sqrt{2n} \| \bmu \|_{\bSigma} \cdot \epsilon_{\lambda} +  c_2 \| \bmu \|_{\bSigma} \sqrt{\log(n)} \cdot \tr(\bSigma)}{ \tr(\bSigma)^2 - \epsilon_{\lambda}^2  }.
\end{align*}
Therefore by the assumption that $\tr(\bSigma) > C \sqrt{n} \epsilon_{\lambda}$ for some large enough absolute constant $C$, we have 
\begin{align*}
    |\bnu^\top \Ab^{-1} \eb_i | &\leq \frac{c_3 \| \bmu \|_{\bSigma}  \cdot  \sqrt{\log(n)}}{ \tr(\bSigma)  }
\end{align*}
for some absolute constant $c_3$. This completes the proof.
\end{proof}

\subsection{Proof of Lemmas in Appendix~\ref{section:proof_appendix1}}\label{section:proof_appendix2}
Here we present the proofs of lemmas we used in Appendix~\ref{section:proof_appendix1}.
\subsubsection{Proof of Lemma~\ref{lemma:eigenvalue_concentration}}
The proof of Lemma~\ref{lemma:eigenvalue_concentration} is motivated by the analysis given in \citet{bartlett2020benign}. However here in Lemma~\ref{lemma:eigenvalue_concentration} we give a slightly tighter bound. The proof is as follows. 

\begin{proof}[Proof of Lemma~\ref{lemma:eigenvalue_concentration}]
Let $\cN$ be a $1/4$-net on the unit sphere $s^{n-1}$. Then by Lemma~5.2 in \citet{vershynin2010introduction}, we have $|\cN| \leq 9^n$. Denote $\zb_j =  \lambda_j^{-1/2} \Qb \vb_j \in \RR^{n}$. Then by definition, for any fixed unit vector $\hat\ab\in \cN$ we have $ \hat\ab^\top \Ab \hat\ab = \Qb\Qb^\top = \hat\ab^\top \sum_{j=1}^p \lambda_j \zb_j \zb_j^\top \hat\ab = \sum_{j=1}^p \lambda_j (\hat\ab^\top\zb_j)^2$. By Lemma 5.9 in \citet{vershynin2010introduction}, there exists an absolute constant $c_1$ such that $\| \hat\ab^\top\zb_j \|_{\psi_2} \leq c_1 \sigma_u$. Therefore by Lemma 21 and Corollary 23 in \citet{bartlett2020benign}, for any $t>0$, with probability at least $1 - 2\exp(-t)$ we have
\begin{align*}
    \big| \hat\ab^\top \Ab \hat\ab - \tr(\bSigma) \big| \leq c_2 \sigma_u^2 \max \big(  t\cdot \| \bSigma\|_{2}  , \sqrt{t}\cdot  \| \bSigma\|_{F}  \big).
\end{align*}
Applying an union bound over all $\hat\ab\in \cN$, we have that with probability at least $1 - 2\cdot 9^n \exp(-t )$,
\begin{align*}
    \big| \hat\ab^\top \Ab \hat\ab - \tr(\bSigma) \big| \leq c_2 \sigma_u^2 \max \big(  t\cdot \| \bSigma\|_{2} , \sqrt{t}\cdot  \| \bSigma\|_{F}  \big) 
\end{align*}
for all $\hat\ab\in \cN$. Therefore by Lemma~25 in \citet{bartlett2020benign}, with probability at least $1 - 2\cdot 9^n \exp(-t )$, we have
\begin{align*}
    \big\| \Ab  - \tr(\bSigma) \Ib \big\|_2 \leq c_3 \sigma_u^2  \big(  t \cdot \| \bSigma\|_{2}  + \sqrt{t}\cdot  \| \bSigma\|_{F}  \big),
\end{align*}
where $c_3$ is an absolute constant. 
Setting $t = c_4 n$ for some large enough constant $c_4$, we have that with probability at least $1 - n^{-2}$,
\begin{align*}
    \big\| \Ab  - \tr(\bSigma) \Ib \big\|_2 \leq c_5  \sigma_u^2 \big(  n \cdot \| \bSigma\|_{2} + \sqrt{n}\cdot  \| \bSigma\|_{F}  \big),
\end{align*}
where $c_5$ is an absolute constant. This completes the proof.
\end{proof}

\subsubsection{Proof of Lemma~\ref{lemma:matrixcalculation}}
Here we present the proof of Lemma~\ref{lemma:matrixcalculation}. Our proof utilizes a key lemma by \citet{wang2020benign}, and gives further simplifications of the result. 
\begin{proof}[Proof of Lemma~\ref{lemma:matrixcalculation}]

Denote $s = \yb^\top \Ab^{-1} \yb$, $t = \bnu^\top \Ab^{-1} \bnu$, $h = \yb^\top \Ab^{-1} \bnu$. Then we have $D =  \| \bmu\|_2^2 s - st + (h+1)^2$. 
By Lemma 3 in \citet{wang2020benign}, we have
\begin{align*}
    \yb^\top (\Xb\Xb^\top)^{-1} &= \yb^\top \Ab^{-1} - D^{-1}\cdot [ \| \bmu \|_2^2 s + h^2 + h - st ] \cdot \yb^\top \Ab^{-1} - D^{-1} s \cdot \bnu^\top \Ab^{-1}.
\end{align*}
Rearranging terms, we obtain
\begin{align*}
    \yb^\top (\Xb\Xb^\top)^{-1} 
    & =\bigg[ 1 - \frac{\| \bmu \|_2^2 s + h^2 + h - st }{ \| \bmu\|_2^2 s - st + (h+1)^2} \bigg] \cdot \yb^\top \Ab^{-1} - D^{-1} s \cdot \bnu^\top \Ab^{-1} \\
    & = \frac{ h + 1 }{ \| \bmu\|_2^2 s - st + (h+1)^2} \cdot \yb^\top \Ab^{-1} - D^{-1} s \cdot \bnu^\top \Ab^{-1}\\
    & = D^{-1} [ (h + 1) \yb^\top \Ab^{-1} - s \cdot \bnu^\top \Ab^{-1}].
\end{align*}
At last, by the definition of $D$, we have
\begin{align*}
     D &= \yb^\top\Ab^{-1} \yb \cdot  (\| \bmu\|_2^2 - \bmu^\top\Qb^\top (\Qb\Qb^\top)^{-1} \Qb\bmu)  + (1 + \yb^\top\Ab^{-1} \bnu)^2\\
     &\geq (1 + \yb^\top\Ab^{-1} \bnu)^2,
\end{align*}
where we utilize the fact that $ \yb^\top\Ab^{-1} \yb \geq 0$ and $\| \bmu\|_2^2 \geq \bmu^\top\Qb^\top (\Qb\Qb^\top)^{-1} \Qb\bmu$. Since $\yb^\top\Ab^{-1} \bnu \neq 1$ with probability $1$, we see that $D > 0$ almost surely. This completes the proof.
\end{proof}

\subsubsection{Proof of Lemma~\ref{lemma:concentrationbounds}}
The proof of Lemma~\ref{lemma:concentrationbounds} is based on the application of eigenvalue concentration results in Lemma~\ref{lemma:eigenvalue_concentration}. We present the details as follows.
\begin{proof}[Proof of Lemma~\ref{lemma:concentrationbounds}]
The bounds on $\yb^\top \Ab^{-1} \yb$ are directly derived from Lemma~\ref{lemma:eigenvalue_concentration} and the fact that $\| \yb \|_2^2 = n$. To derive the bounds for $\bnu\Ab^{-1}\bnu$, we note that by definition, $\bnu = \Qb\bmu$ and 
\begin{align*}
    \bnu^\top \Ab^{-1} \bnu = \bmu^\top \Qb^\top (\Qb\Qb^\top)^{-1} \Qb \bmu. 
\end{align*}

Denote $\zb_i =  \lambda_i^{-1/2} \Qb \vb_i \in \RR^{n}$,  $\Zb = [\zb_1,\ldots, \zb_p] \in \RR^{n \times p}$, and $\tilde\bmu = \Lambda^{1/2} \Vb^\top \bmu$. Then $\Qb = \Zb \bLambda^{1/2} \Vb^\top$, $\Qb \bmu = \Zb \tilde\bmu$, and
\begin{align*}
    \bmu^\top \Qb^\top (\Qb\Qb^\top)^{-1} \Qb \bmu &= \bmu^\top \Vb  \bLambda^{1/2} \Zb^\top ( \Zb \bLambda \Zb^\top )^{-1} \Zb \bLambda^{1/2} \Vb^\top \bmu\\
    &= \tilde\bmu^\top  \Zb^\top ( \Zb \bLambda \Zb^\top )^{-1} \Zb  \tilde\bmu\\
    &\leq \frac{ \| \Zb \tilde\bmu \|_2^2 }{ \tr(\bSigma) - \epsilon_{\lambda} }.
\end{align*}
Similarly, we have
\begin{align*}
    \bmu^\top \Qb^\top (\Qb\Qb^\top)^{-1} \Qb \bmu \geq \frac{ \| \Zb \tilde\bmu \|_2^2 }{  \tr(\bSigma) + \epsilon_{\lambda} }.
\end{align*}
We now proceed to give upper and lower bounds for the term $\| \Zb \tilde\bmu \|_2^2 =  \sum_{i=1}^n ( \sum_{j=1}^p \Zb_{ij} \tilde\mu_j )^2$. Note that by definition,  $\Zb_{ij}$ for $i\in [n]$ and $j\in [p]$ are independent sub-Gaussian vectors with $\| \Zb_{ij} \|_{\psi_2} \leq \sigma_u$. By Lemma~5.9 in \citet{vershynin2010introduction}, we have
\begin{align*}
    \Bigg\| \sum_{j=1}^p \Zb_{ij} \tilde\mu_j \Bigg\|_{\psi_2} \leq c_1 \| \tilde\bmu \|_2 \cdot \sigma_u, 
\end{align*}
where $c_1$ is an absolute constant. Therefore by Lemma~5.14 in \citet{vershynin2010introduction}, we have
\begin{align*}
     \Bigg\| \Bigg(\sum_{j=1}^p \Zb_{ij} \tilde\mu_j \Bigg)^2 - \| \tilde\bmu \|_2^2 \Bigg\|_{\psi_1} \leq  c_2 \| \tilde\bmu_j \|_2^2,
\end{align*}
where we merge $\sigma_u$ into the absolute constant $c_2$. 
By Bernstein's inequality, with probability at least $1 - n^{-2}$, 
\begin{align*}
    \big| \| \Zb \tilde\bmu \|_2^2 - \EE \| \Zb \tilde\bmu \|_2^2 \big| \leq c_3 \| \tilde\bmu \|_2^2\cdot \sqrt{n\log( n)},
\end{align*}
where $c_3$ is an absolute constant. Therefore we have
\begin{align}\label{eq:concentrationbounds_eq1}
     n  \| \tilde\bmu \|_2^2 - c_3 \| \tilde\bmu \|_2^2\cdot \sqrt{n\log( n)} \leq  \| \Qb \bmu \|_2^2 = \| \Zb \tilde\bmu \|_2^2 \leq  n \| \tilde\bmu \|_2^2 + c_3 \| \tilde\bmu \|_2^2\cdot \sqrt{n\log( n)},
\end{align}
and
\begin{align*}
    \frac{n - c_3 \sqrt{n\log( n)}}{ \tr(\bSigma) + \epsilon_{\lambda} } \cdot \| \tilde\bmu \|_2 \leq  \bnu^\top \Ab^{-1} \bnu \leq \frac{n + c_3 \sqrt{n\log( n)}}{ \tr(\bSigma) - \epsilon_{\lambda} } \cdot \| \tilde\bmu \|_2
\end{align*}
Similarly for $\yb^\top \Ab^{-1} \bnu$, by Cauchy-Schwarz inequality, for large enough $n$ we have
\begin{align*}
    |\yb^\top \Ab^{-1} \bnu| = |\yb^\top (\Qb\Qb^\top)^{-1} \Qb \bmu| \leq \| \yb \|_2 \cdot \| (\Qb\Qb^\top)^{-1} \Qb \bmu \|_2 = \sqrt{n} \cdot \sqrt{ \bmu^\top \Qb^\top (\Qb\Qb^\top)^{-2} \Qb \bmu }.
\end{align*}
Applying Lemma~\ref{lemma:eigenvalue_concentration} and the inequality \eqref{eq:concentrationbounds_eq1}, we have
\begin{align*}
    |\yb^\top \Ab^{-1} \bnu|
    &\leq\frac{\sqrt{n} }{  \tr(\bSigma) - \epsilon_{\lambda}  } \| \Qb\bmu \|_2 \leq\frac{\sqrt{n} \cdot \sqrt{n + c_3 \sqrt{n\log( n)}} }{  \tr(\bSigma) - \epsilon_{\lambda}  } \| \tilde\bmu \|_2 \leq \frac{c_4 n }{  \tr(\bSigma) - \epsilon_{\lambda}  } \| \tilde\bmu \|_2,
\end{align*}
where $c_4$ is an absolute constant. Note that $\| \tilde\bmu \|_2 = \| \bmu \|_{\bSigma}$. This completes the proof.
\end{proof}

\section{Proof of Theorem~\ref{thm:BO_lowerbound}}
Here we present the proof of Theorem~\ref{thm:BO_lowerbound}. 
\begin{proof}[Proof of Theorem~\ref{thm:BO_lowerbound}]
By the lower bound of the Gaussian cumulative distribution function \citep{cote2012chernoff}, we have that for any $\btheta \in \RR^d$,
\begin{align}\label{eq:lowerboundproof_eq1}
    R(\btheta) \geq c_1 \exp\bigg( - \frac{ c_2( \btheta^\top \bmu )^2 }{ \|\btheta\|_{\bSigma}^2} \bigg),
\end{align}
where $c_1,c_2 > 0$ are absolute constants. By Proposition~\ref{prop:interpolationregression}, we have
\begin{align*}
    \hat\btheta_{\text{SVM}} = \hat\btheta_{\text{LS}}  = \Xb^\top (\Xb\Xb^\top)^{-1} \yb.
\end{align*}
Plugging it into \eqref{eq:lowerboundproof_eq1}, we obtain
\begin{align}\label{eq:lowerboundproof_eq2}
    R(\hat\btheta_{\text{SVM}} ) \geq c_1 \exp\bigg\{ - \frac{ c_2 [ \yb^\top (\Xb\Xb^\top)^{-1} \Xb \bmu  ]^2}{ \| \Xb^\top (\Xb\Xb^\top)^{-1} \yb \|_{\bSigma}^2 } \bigg\}.
\end{align}
Note that based on our model, we have $\Xb = \yb \bmu^\top + \Qb$, and 
\begin{align}
    \| \Xb^\top (\Xb\Xb^\top)^{-1} \yb \|_{\bSigma}
    &\quad= \| (\yb \bmu^\top + \Qb)^\top (\Xb\Xb^\top)^{-1} \yb \|_{\bSigma}  \nonumber\\
    &\quad\geq \big| \| \bmu \yb^\top (\Xb\Xb^\top)^{-1} \yb \|_{\bSigma} - \| \Qb^\top (\Xb\Xb^\top)^{-1} \yb \|_{\bSigma} \big| \nonumber\\
    &\quad=\big| \yb^\top (\Xb\Xb^\top)^{-1} \yb\cdot \| \bmu  \|_{\bSigma} - \| \Qb^\top (\Xb\Xb^\top)^{-1} \yb \|_{\bSigma} \big| 
\end{align}
Plugging the above bound into \eqref{eq:lowerboundproof_eq2}, we obtain
\begin{align}\label{eq:lowerboundproof_eq3}
    R(\btheta) \geq c_1 \exp\bigg\{ - \frac{ c_2[ \yb^\top (\Xb\Xb^\top)^{-1} \Xb \bmu  ]^2 }{  (\yb^\top (\Xb\Xb^\top)^{-1} \yb \cdot\| \bmu \|_{\bSigma} - \| \Qb^\top (\Xb\Xb^\top)^{-1} \yb \|_{\bSigma})^2} \bigg\}.
\end{align}
Denote $\bnu = \Qb \bmu$ and $\Ab = \Qb\Qb^\top$. Then by Lemma~\ref{lemma:matrixcalculation} and the model definition $\Xb = \yb \bmu^\top + \Qb$, we have
\begin{align}
    \yb^\top (\Xb\Xb^\top)^{-1}  \Xb \bmu &= D^{-1} [ ( 1 + \yb^\top \Ab^{-1} \bnu ) \yb^\top \Ab^{-1} - \yb^\top \Ab^{-1} \yb \cdot \bnu^\top \Ab^{-1}] (\yb \bmu^\top + \Qb)\bmu \nonumber\\
    &= D^{-1} [ ( 1 + \yb^\top \Ab^{-1} \bnu ) \yb^\top \Ab^{-1} - \yb^\top \Ab^{-1} \yb \cdot \bnu^\top \Ab^{-1}] (\yb \cdot \| \bmu \|_2^2 + \Qb\bmu) \nonumber\\
    & = D^{-1} [ ( 1 + \yb^\top \Ab^{-1} \bnu ) \yb^\top \Ab^{-1} \yb - \yb^\top \Ab^{-1} \yb \cdot \bnu^\top \Ab^{-1} \yb] \cdot \| \bmu \|_2^2 \nonumber \\
    & \qquad + D^{-1} [ ( 1 + \yb^\top \Ab^{-1} \bnu ) \yb^\top \Ab^{-1}\bnu - \yb^\top \Ab^{-1} \yb \cdot \bnu^\top \Ab^{-1}\bnu] )\nonumber\\
    & = D^{-1}\cdot [ (\| \bmu \|_2^2 - \bnu^\top \Ab^{-1}\bnu) \yb^\top \Ab^{-1} \yb  + ( 1 + \yb^\top \Ab^{-1} \bnu ) \yb^\top \Ab^{-1}\bnu  ],\label{eq:refinedBO_numerator_eq0_lowerbound}
\end{align}
where the third equality follows by the notation $\bnu = \Qb\bmu$.
By Lemma~\ref{lemma:concentrationbounds} and the assumption that $\tr( \bSigma ) \geq C \max\{ \epsilon_{\lambda} , n \|\bSigma \|_2 , n \| \bmu \|_{\bSigma} \}$ for some large enough constant $C$, when $n$ is large enough we have
\begin{align*}
     &| \yb^\top \Ab^{-1} \bnu | \leq \frac{ c_3 n   }{  \tr( \bSigma )  - \epsilon_{\lambda}  } \| \bmu \|_{\bSigma} \leq \frac{ 2c_4 n   }{  \tr( \bSigma ) } \| \bmu \|_{\bSigma} \leq 1,\\
     &0\leq \bnu^\top \Ab^{-1}\bnu \leq \frac{ n  + c_5 \sqrt{n\log( n)} }{ \tr( \bSigma )  - \epsilon_{\lambda} } \cdot \| \bmu \|_{\bSigma}^2 \leq \frac{ 2n }{ \tr( \bSigma ) } \cdot \| \bmu \|_{\bSigma}^2 \leq \frac{ 2n \| \bSigma \|_2 }{ \tr( \bSigma ) } \cdot \| \bmu \|_2^2 \leq \frac{1}{2}\cdot  \| \bmu \|_2^2 ,\\
     & 0\leq \yb^\top \Ab^{-1} \yb \leq \frac{n}{ \tr( \bSigma )  - \epsilon_{\lambda} } \leq \frac{2n}{ \tr( \bSigma ) },
\end{align*}
where $c_3,c_4$ are absolute constants. 
Plugging the bounds above into \eqref{eq:refinedBO_numerator_eq0}, we obtain
\begin{align*}
    | \yb^\top (\Xb\Xb^\top)^{-1}  \Xb \bmu | &\leq D^{-1}\cdot  \bigg(\| \bmu \|_2^2\cdot \yb^\top \Ab^{-1} \yb + 2\cdot   |\yb^\top \Ab^{-1}\bnu | \bigg)\\
    &\leq D^{-1}\cdot \bigg[ \frac{2n}{ \tr( \bSigma ) } \cdot \| \bmu \|_2^2 + \frac{ 4n   }{  \tr( \bSigma ) } \| \bmu \|_{\bSigma} \bigg]\\
    &\leq D^{-1}\cdot \frac{2n}{ \tr( \bSigma ) } \cdot  ( \| \bmu \|_2^2 + 2 \| \bmu \|_{\bSigma} )\\
    &\leq D^{-1}\cdot \frac{4n}{ \tr( \bSigma ) } \cdot  \| \bmu \|_2^2,
\end{align*}
where the last inequality follows by the assumption that $\| \bmu \|_2^2 \geq C \| \bmu\|_{\bSigma} $ for some large enough absolute constant $C$. 
Therefore we have
\begin{align}\label{eq:refinedBO_numerator_upperbound}
    [\yb^\top (\Xb\Xb^\top)^{-1}  \Xb \bmu]^2 \leq D^{-2}\cdot \frac{n^2}{ 64 [\tr( \bSigma )]^2 } \cdot   \| \bmu \|_2^4 = \frac{H(\bmu, \Qb, \yb, \bSigma) }{64} \cdot n^2 \| \bmu \|_2^4,
\end{align}
where 
$$ H(\bmu, \Qb, \yb, \bSigma) := [D\cdot \tr(\bSigma) ]^{-2} > 0. $$
We now proceed to study the two terms in the denominator of the exponent in \eqref{eq:lowerboundproof_eq3}. We denote
\begin{align*}
    &J_1 = \yb^\top (\Xb\Xb^\top)^{-1} \yb\cdot \| \bmu  \|_{\bSigma},\\
    &J_2 = \| \Qb^\top (\Xb\Xb^\top)^{-1} \yb \|_{\bSigma})^2.
\end{align*}
Then for $J_1$, with the same derivation as the proof of Lemma~\ref{lemma:I1I2I3bounds} for $I_2$, we have
\begin{align*}
    J_1 = \sqrt{I_2} \leq 2 \sqrt{H(\bmu, \Qb, \yb, \bSigma)} \cdot n \cdot \| \bmu \|_{\bSigma}.
\end{align*}
Moreover we also have
\begin{align*}
    \yb^\top (\Xb\Xb^\top)^{-1} \yb = D^{-1} \cdot \yb^\top \Ab^{-1} \yb \geq D^{-1} \cdot \frac{n}{ \tr( \bSigma )  + \epsilon_{\lambda} } \geq  (2D)^{-1}  \cdot \frac{n}{ \tr( \bSigma )},
\end{align*}
where the first inequality follows by Lemma~\ref{lemma:concentrationbounds}, and the second inequality follows by the assumption that $\tr( \bSigma ) \geq C \epsilon_{\lambda}$ for some large enough constant $C$. Then we have
\begin{align*}
    J_1 = \yb^\top (\Xb\Xb^\top)^{-1} \yb \cdot  \| \bmu \|_{\bSigma} \geq  (2D)^{-1}  \cdot \frac{n \cdot \| \bmu \|_{\bSigma}}{ \tr( \bSigma )} = (1/2)\cdot \sqrt{H(\bmu, \Qb, \yb, \bSigma)} \cdot n \cdot \| \bmu \|_{\bSigma},
\end{align*}
where we use the definition $H(\bmu, \Qb, \yb, \bSigma) =  [D\cdot \tr(\bSigma) ]^{-2}$. Therefore in summary we have
\begin{align}\label{eq:lowerboundproof_J1bounds}
    (1/2)\cdot \sqrt{H(\bmu, \Qb, \yb, \bSigma)} \cdot n \cdot \| \bmu \|_{\bSigma} \leq J_1 \leq 2 \sqrt{H(\bmu, \Qb, \yb, \bSigma)} \cdot n \cdot \| \bmu \|_{\bSigma},
\end{align}
where $c_5$ is an absolute constant. 
Similarly, for $J_2$, with the same derivation as the proof of Lemma~\ref{lemma:I1I2I3bounds} for $I_3$, we have
\begin{align}\label{eq:lowerboundproof_J2upperbound}
    J_2^2 = I_3 \leq c_5 H(\bmu, \Qb, \yb, \bSigma)\cdot ( n\cdot \|\bSigma\|_F^2 +  n^2\cdot \| \bSigma\|_2^2 ).
\end{align}
Moreover, we denote $\ab = D^{-1}[( 1 + \yb^\top \Ab^{-1} \bnu )\cdot \yb - \yb^\top \Ab^{-1} \yb \cdot \bnu]$. Then with the same derivation,
\begin{align}
    J_2^2 &= \yb^\top (\Xb\Xb^\top)^{-1} \Qb \bSigma \Qb^\top (\Xb\Xb^\top)^{-1} \yb \nonumber \\
    &= \ab^\top (\Qb\Qb^\top)^{-1}\Qb\bSigma\Qb^\top (\Qb\Qb^\top)^{-1} \ab \nonumber \\
    & = \ab^\top (\Zb\bLambda \Zb^\top)^{-1} \Zb\bLambda^2 \Zb^\top (\Zb\bLambda \Zb^\top)^{-1} \ab, \label{eq:refinedBOderivation_lowerbound_eq1}
\end{align}
where we plug in $\bSigma = \Vb \bLambda \Vb^\top$ and $\Qb = \Zb \bLambda^{1/2} \Vb^\top$ for $\Zb$ with independent sub-Gaussian entries.
We have
\begin{align}
    J_2^2&= \ab^\top (\Zb\bLambda \Zb^\top)^{-1} \Zb\bLambda^2 \Zb^\top (\Zb\bLambda \Zb^\top)^{-1} \ab \nonumber\\
    & \geq \ab^\top (\Zb\bLambda \Zb^\top)^{-2} \ab \cdot \big[\|\bSigma\|_F^2 - \epsilon_{\lambda}'\big] \nonumber\\
    & \geq \| \ab \|_2^2\cdot \frac{\|\bSigma\|_F^2 - \epsilon_{\lambda}'}{ [\tr(\bSigma) + \epsilon_{\lambda}]^2}\nonumber\\
    &\geq  \| \ab \|_2^2\cdot \frac{\|\bSigma\|_F^2 - \epsilon_{\lambda}'}{2 [\tr(\bSigma) ]^2}.
    \label{eq:refinedBOderivation_lowerbound_eq1.5}
\end{align}
Here the first inequality follows by Lemma~\ref{lemma:eigenvalue_concentration2}, the second inequality follows by Lemma~\ref{lemma:eigenvalue_concentration}, and the third inequality follows by the assumption that $\tr(\bSigma) \geq C \epsilon_{\lambda}$ for some large enough absolute constant $C$.  
By the definition of $\epsilon_{\lambda}'$ in Lemma~\ref{lemma:eigenvalue_concentration2} and Cauchy-Schwarz inequality, we have
\begin{align*}
    \epsilon_{\lambda}' &:= c_6  \big( n\cdot \| \bSigma\|_2^2  + \sqrt{n}\cdot \| \bSigma^2\|_F \big)\\
    &\leq c_6  \big( n\cdot \| \bSigma\|_2^2  + \sqrt{n}\cdot \| \bSigma\|_2 \cdot \| \bSigma\|_F \big)\\
    &\leq c_6  \big( n\cdot \| \bSigma\|_2^2  + 2c_6n\cdot \| \bSigma\|_2^2  +  \| \bSigma\|_F^2 / (2c_6) \big)\\
    &\leq c_7  n\cdot \| \bSigma\|_2^2  +  \| \bSigma\|_F^2  / 2,
\end{align*}
where $c_6$, $c_7$ are absolute constants. Plugging the above bound into \eqref{eq:refinedBOderivation_lowerbound_eq1.5} gives
\begin{align}
    J_2^2 \geq  \| \ab \|_2^2\cdot \frac{\|\bSigma\|_F^2 - c_8 n\cdot \| \bSigma\|_2^2 }{ 4 [\tr(\bSigma) ]^2}
    \label{eq:refinedBOderivation_lowerbound_eq1.6}
\end{align}
for some absolute constant $c_8$.
 Moreover, by the definition $\ab$ and the triangle inequality, we have
\begin{align}
    \| \ab \|_2^2 & = \| D^{-1}( 1 + \yb^\top \Ab^{-1} \bnu ) \yb - \yb^\top \Ab^{-1} \yb \cdot \bnu \|_2^2 \nonumber \\
    & \geq \big[ D^{-1} ( 1 + \yb^\top \Ab^{-1} \bnu ) \| \yb \|_2 -  D^{-1} ( \yb^\top \Ab^{-1} \yb  )\cdot  \|  \Qb \bmu\|_2\big] ^2 \nonumber\\
    &= D^{-2}\big[  ( 1 + \yb^\top \Ab^{-1} \bnu ) \cdot \sqrt{n} -   ( \yb^\top \Ab^{-1} \yb  )\cdot  \|  \Qb \bmu\|_2\big] ^2. \label{eq:refinedBOderivation_lowerbound_eq1.7}
\end{align}
By Lemma~\ref{lemma:concentrationbounds} and the assumption that $\tr( \bSigma ) \geq C \max\{ \epsilon_{\lambda} , n, n \| \bmu \|_{\bSigma} \}$ for some large enough constant $C$, we have
\begin{align*}
    &|\yb^\top \Ab^{-1} \bnu| \leq \frac{c_9 n   }{  \tr( \bSigma )  - \epsilon_{\lambda}  } \| \bmu \|_{\bSigma} \leq 1/2, \\
    &  \yb^\top \Ab^{-1} \yb \leq \frac{n}{ \tr( \bSigma )  - \epsilon_{\lambda} } \leq \frac{2n}{\tr( \bSigma )},
\end{align*}
where $c_9$ is an absolute constant. Moreover, with the same proof as in Lemma~\ref{lemma:concentrationbounds}, when $n$ is sufficiently large, with probability at least $1- O(n^{-2})$ we have 
\begin{align*}
      \| \Qb \bmu \|_2^2 \leq  2n \| \bmu \|_{\bSigma}^2.
\end{align*}
Utilizing these inequalities above, we have
\begin{align*}
      &( 1 + \yb^\top \Ab^{-1} \bnu ) \cdot \sqrt{n} \geq \sqrt{n} /2,\\
      &(\yb^\top \Ab^{-1} \yb  )\cdot  \|  \Qb \bmu\|_2 \leq  \frac{2n}{\tr( \bSigma )}\cdot \sqrt{2n} \| \bmu \|_{\bSigma} \leq \sqrt{n} /4,
\end{align*}
where the second line above follows the assumption that $\tr( \bSigma ) \geq C n \| \bmu \|_{\bSigma}$ for some large enough constant $C$. Combining these bounds with \eqref{eq:refinedBOderivation_lowerbound_eq1.7}, we have
\begin{align*}
    \| \ab \|_2^2 &\geq D^{-2}\big[  ( 1 + \yb^\top \Ab^{-1} \bnu ) \cdot \sqrt{n} -   ( \yb^\top \Ab^{-1} \yb  )\cdot  \|  \Qb \bmu\|_2\big] ^2  \geq  D^{-2} n / 16.
\end{align*}
Further plugging this bound into \eqref{eq:refinedBOderivation_lowerbound_eq1.6}, we have
\begin{align}\label{eq:lowerboundproof_J2lowerbound}
    J_2^2 \geq \frac{ n }{16D^{2}}\cdot \frac{\|\bSigma\|_F^2 - c_8 n\cdot \| \bSigma\|_2^2 }{ 4 [\tr(\bSigma) ]^2} = H(\bmu, \Qb, \yb, \bSigma) \cdot (c_{10}n\cdot\|\bSigma\|_F^2 - c_{11} n^2\cdot \| \bSigma\|_2^2),
\end{align}
where $c_{10},c_{11}$ are absolute constants, and we use the definition $H(\bmu, \Qb, \yb, \bSigma) =  [D\cdot \tr(\bSigma) ]^{-2}$. Combining \eqref{eq:lowerboundproof_J2upperbound} and \eqref{eq:lowerboundproof_J2lowerbound}, we obtain
\begin{align}\label{eq:lowerboundproof_J2bounds}
    H(\bmu, \Qb, \yb, \bSigma) \cdot (c_{10}n\|\bSigma\|_F^2 - c_{11} n^2 \| \bSigma\|_2^2) \leq J_2^2 \leq c_5 H(\bmu, \Qb, \yb, \bSigma)\cdot ( n\|\bSigma\|_F^2 +  n^2\| \bSigma\|_2^2 ).
\end{align}

In the rest of the proof, we consider the two cases in Theorem~\ref{thm:BO_lowerbound} separately based on \eqref{eq:lowerboundproof_J1bounds} and \eqref{eq:lowerboundproof_J2bounds}. 

\noindent\textbf{Case 1.} Suppose that  $ n \| \bmu \|_{\bSigma}^2 \geq C (\| \bSigma \|_F^2 + n \| \bSigma\|_2^2)$ for some large enough constant $C$. Then by \eqref{eq:lowerboundproof_J1bounds} and \eqref{eq:lowerboundproof_J2bounds}, we have
\begin{align*}
    &J_1 \geq (1/2)\cdot \sqrt{H(\bmu, \Qb, \yb, \bSigma)} \cdot n \cdot \| \bmu \|_{\bSigma}\\
    &J_2 \leq 2\sqrt{c_5} \sqrt{H(\bmu, \Qb, \yb, \bSigma)}\cdot \sqrt{ n\cdot \|\bSigma\|_F^2 +  n^2\cdot \| \bSigma\|_2^2} \leq (1/4)\cdot \sqrt{H(\bmu, \Qb, \yb, \bSigma)} \cdot n \cdot \| \bmu \|_{\bSigma}
\end{align*}
Plugging the above inequalities and \eqref{eq:refinedBO_numerator_upperbound} into \eqref{eq:lowerboundproof_eq3}, we obtain
Therefore by 
\begin{align*}
    R(\btheta) \geq c_1 \exp\bigg\{ - \frac{ c_2 n^2 \| \bmu \|_2^4 / 64}{  (n \cdot \| \bmu \|_{\bSigma} / 4)^2} \bigg\} = c_1 \exp\bigg\{ - \frac{ c_{12}\| \bmu \|_2^4 }{  \| \bmu \|_{\bSigma}^2} \bigg\},
\end{align*}
where $c_{12}$ is an absolute constant. 
This completes the proof of the first case in Theorem~\ref{thm:BO_lowerbound}.

\noindent\textbf{Case 2.} Suppose that  $\| \bSigma \|_F^2 \geq C n (\| \bmu \|_{\bSigma}^2 + \| \bSigma\|_2^2)$ for some large enough constant $C$. Then by \eqref{eq:lowerboundproof_J1bounds} we have
\begin{align}\label{eq:lowerboundproof_case2_eq1}
    J_1 \leq 2  \sqrt{H(\bmu, \Qb, \yb, \bSigma)} \cdot n \cdot \| \bmu \|_{\bSigma} \leq \sqrt{H(\bmu, \Qb, \yb, \bSigma)} \cdot \sqrt{c_{10}}\|\bSigma\|_F / 4.
\end{align}
Moreover for $J_2$, by \eqref{eq:lowerboundproof_J2bounds} we have
\begin{align*}
    J_2^2 \geq H(\bmu, \Qb, \yb, \bSigma) \cdot (c_{10} n \|\bSigma\|_F^2 - c_{11} n^2 \| \bSigma\|_2^2) \geq H(\bmu, \Qb, \yb, \bSigma) \cdot c_{10}n \|\bSigma\|_F^2 / 4,
\end{align*}
and therefore 
\begin{align}\label{eq:lowerboundproof_case2_eq2}
    J_2 \geq \sqrt{H(\bmu, \Qb, \yb, \bSigma)} \cdot \sqrt{c_{10} n }\|\bSigma\|_F / 2.
\end{align}
Plugging \eqref{eq:refinedBO_numerator_upperbound}, \eqref{eq:lowerboundproof_case2_eq1} and \eqref{eq:lowerboundproof_case2_eq2} into \eqref{eq:lowerboundproof_eq3}, we obtain
\begin{align*}
    R(\btheta) \geq c_1 \exp\bigg\{ - \frac{ c_2 n^2 \| \bmu \|_2^4 / 64}{  (\sqrt{c_{10} n }\|\bSigma\|_F / 4)^2} \bigg\} = c_1 \exp\bigg\{ - \frac{ c_{13} n \| \bmu \|_2^4 }{  \|\bSigma\|_F^2} \bigg\},
\end{align*}
where $c_{13}$ is an absolute constant. 
This completes the proof of the second case in Theorem~\ref{thm:BO_lowerbound}. 
\end{proof}

\section{Proof of Corollaries}\label{section:corollaryproof}
Here we provide the proof of the Corollaries~\ref{col:isotropic}, \ref{col:polynomialdecay} and \ref{col:rare-weak} in Section~\ref{section:mainresults}.

\subsection{Proof of Corollary~\ref{col:isotropic}}
The proof of Corollary~\ref{col:isotropic} is a direct application of Theorem~\ref{thm:BOnew}. The detailed proof is as follows.

\begin{proof}[Proof of Corollary~\ref{col:isotropic}]
When $\bSigma = \Ib$, we have $\tr( \bSigma ) = d$, $\|\bSigma \|_2 = 1$, $\| \bSigma \|_F = \sqrt{d}$ and $\| \bmu \|_{\bSigma} = \| \bmu \|_2$. Under the condition in  Corollary~\ref{col:isotropic} that $d \geq C \max\big\{ n^{2} , n\sqrt{\log(n)}\cdot \| \bmu \|_{2} \big\}$ and $\| \bmu \|_2 \geq C $ for some large enough absolute constant $C$, it is easy to check that the conditions of Theorem~\ref{thm:BOnew}
$$\tr( \bSigma ) = \Omega\big( \max\big\{ n^{3/2} \|\bSigma \|_2, n\| \bSigma \|_F , n\sqrt{\log(n)}\cdot \| \bmu \|_{\bSigma} \big\} \big),~\| \bmu \|_2 \geq C \| \bSigma\|_{2}$$ 
hold. Therefore by Theorem~\ref{thm:BOnew}, we have
\begin{align*}
    R(\hat\btheta_{\text{SVM}}) \leq \exp\Bigg(  \frac{ - c_1 n \| \bmu \|_2^4  }{  n  \| \bmu \|_{\bSigma}^2+ \| \bSigma\|_F^2 +  n\| \bSigma\|_2^2 } \Bigg) \leq \exp\Bigg(  \frac{ - c_2 n \| \bmu \|_2^4  }{  n  \| \bmu \|_2^2+ d} \Bigg),
\end{align*}
where $c_1$, $c_2$ are absolute constants. This completes the proof. 
\end{proof}

\subsection{Proof of Corollary~\ref{col:polynomialdecay}}
Here we present the proof of Corollary~\ref{col:polynomialdecay}, which is mostly based on the estimation of the order of the summations $\sum_{k=1}^d k^{-\alpha}$ and $\sum_{k=1}^d k^{-2\alpha}$. We first present the full version of the corollary with detailed dependency in the sample size $n$ as follows.

\begin{corollary}\label{col:polynomialdecay_with_n}[Full version of Corollary~\ref{col:polynomialdecay}]
    Suppose that $\lambda_k = k^{-\alpha}$, and one of the following conditions hold:
\begin{enumerate}[leftmargin = *]
    \item $\alpha\in [0, 1/2)$, $ d = \tilde\Omega( n^{\frac{3}{2(1 - \alpha)}} + n^2 + ( n \| \bmu \|_{\bSigma})^{\frac{1}{1 - \alpha}}  )$, and $\| \bmu \|_2 = \omega( 1 + n^{-1/4}d^{1/4 - \alpha/2}  )$.
    \item $\alpha = 1/2$, $ d = \tilde\Omega( n^{3} + n^2 \| \bmu \|_{\bSigma}^2  )$, and $\| \bmu \|_2 = \omega( 1 + n^{-1/4}(\log(d))^{1/4} )$.
    \item $\alpha\in (1/2, 1 )$, $ d = \tilde\Omega( n^{\frac{3}{2(1 - \alpha)}} + ( n \| \bmu \|_{\bSigma})^{\frac{1}{1 - \alpha}}  )$, and $\| \bmu \|_2 = \omega(1)$.
\end{enumerate}
Then with probability at least $1 - n^{-1}$, the population risk of the maximum margin classifier satisfies $R(\hat\btheta_{\text{SVM}}) = o(1)$.
\end{corollary}

\begin{proof}[Proof of Corollary~\ref{col:polynomialdecay_with_n}]
We first consider the case when $\alpha\in [0, 1/2)$. We have
\begin{align*}
    \tr(\bSigma) = \sum_{k=1}^d \lambda_k = \sum_{k=1}^d k^{-\alpha} \geq \int_{t=1}^{d} t^{-\alpha} \dd t = \frac{d^{1 - \alpha}}{1 - \alpha} - \frac{1}{1 - \alpha} > \frac{d^{1 - \alpha}}{2(1 - \alpha)}
\end{align*}
when $d$ is sufficiently large. Similarly, we have
\begin{align*}
    \| \bSigma \|_F^2 = \sum_{k=1}^d \lambda_k^2 = 1 + \sum_{k=2}^d k^{-2\alpha} \leq 1 + \int_{t=1}^{d-1} t^{-2\alpha} \dd t = 1 + \frac{(d - 1)^{1 - 2\alpha}}{1 - 2\alpha} - \frac{1}{1 - 2\alpha} \leq 1 + \frac{d^{1 - 2\alpha}}{1 - 2\alpha}.
\end{align*}
Therefore, a sufficient condition for the assumptions in Theorem~\ref{thm:BOnew} to hold is that $\| \bmu \|_2 = \omega(1)$ and
\begin{align*}
    &\frac{d^{1 - \alpha}}{2(1 - \alpha)}  \geq C n^{3/2},\\
    &\frac{d^{1 - \alpha}}{2(1 - \alpha)}  \geq C n \cdot \sqrt{1 + \frac{d^{1 - 2\alpha}}{1 - 2\alpha} },\\
    &\frac{d^{1 - \alpha}}{2(1 - \alpha)}  \geq C n\sqrt{\log(n)}\cdot \| \bmu \|_{\bSigma} .
\end{align*}
After simplifying the result, we derive the condition that  $ d = \tilde\Omega( n^{\frac{3}{2(1 - \alpha)}} + n^2 + ( n \| \bmu \|_{\bSigma})^{\frac{1}{1 - \alpha}}  )$. We further check the conditions on $\| \bmu \|_2$ that lead to $o(1)$ population risk. Note that when $\| \bmu \|_2 = \omega(1)$, $\| \bmu \|_2^4 / \| \bmu \|_{\bSigma}^2 = \omega(1)$. We also check the condition that $n \| \bmu \|_2^4 / \| \bSigma \|_F^2 = \omega(1)$. A sufficient condition is that
\begin{align*}
    n \| \bmu \|_2^4 = \omega\bigg( 1 + \frac{d^{1 - 2\alpha}}{1 - 2\alpha} \bigg).
\end{align*}
Simplifying the condition completes the proof for the case $\alpha\in [0, 1/2)$.

For the case $\alpha = 1/2$, we have
\begin{align*}
    \tr(\bSigma) = \sum_{k=1}^d \lambda_k = \sum_{k=1}^d k^{-1/2} \geq \int_{t=1}^{d} t^{-1/2} \dd t = \frac{d^{1 - 1/2}}{1 - 1/2} - \frac{1}{1 - 1/2} > \sqrt{d}
\end{align*}
when $d$ is sufficiently large. Moreover,
\begin{align*}
    \| \bSigma \|_F^2 = \sum_{k=1}^d \lambda_k^2 = 1 + \sum_{k=2}^d k^{-1} \leq 1 + \int_{t=1}^{d-1} t^{-1} \dd t = 1 + \log(d - 1) \leq 1 + \log(d).
\end{align*}
Verifying the conditions
\begin{align*}
    &\sqrt{d}  \geq C n^{3/2},\\
    &\sqrt{d} \geq C n \cdot \sqrt{1 + \log(d) },\\
    &\sqrt{d} \geq C n\sqrt{\log(n)}\cdot \| \bmu \|_{\bSigma}
\end{align*}
then gives a sufficient condition $ d = \tilde\Omega( n^{3} + n^2 \| \bmu \|_{\bSigma}^2  )$, $\| \bmu \|_2 = \omega(1)$ for the assumptions in Theorem~\ref{thm:BOnew} to hold. It is also easy to verify that when $\| \bmu \|_2 = \omega( 1 + n^{-1/4}(\log(d))^{1/4} )$ we have $R(\hat\btheta_{\text{SVM}}) = o(1)$.

Finally for the case $\alpha\in (1/2, 1 )$, we have
\begin{align*}
    \tr(\bSigma) = \sum_{k=1}^d \lambda_k = \sum_{k=1}^d k^{-\alpha} \geq \int_{t=1}^{d} t^{-\alpha} \dd t = \frac{d^{1 - \alpha}}{1 - \alpha} - \frac{1}{1 - \alpha}.
\end{align*}
Moreover, in this setting we have $\| \bSigma \|_F^2 \leq c_1$ for some absolute constant $c_1$. It is therefore easy to check that $\| \bmu \|_2 = \omega(1)$ and
\begin{align*}
    d = \tilde\Omega( n^{\frac{3}{2(1 - \alpha)}} + ( n \| \bmu \|_{\bSigma})^{\frac{1}{1 - \alpha}}  )
\end{align*}
are sufficient for the assumptions in Theorem~\ref{thm:BOnew} to hold, and we also have $R(\hat\btheta_{\text{SVM}}) = o(1)$.
\end{proof}

\subsection{Proof of Corollary~\ref{col:rare-weak}}
The proof of Corollary~\ref{col:rare-weak} for the rare/weak feature model is rather straightforward. 
\begin{proof}[Proof of Corollary~\ref{col:rare-weak}]
Note that in the rare/weak feature model we have $\| \bmu \|_2 = \gamma \sqrt{s}$. Therefore the conditions of Corollary~\ref{col:isotropic} are satisfied and we have 
\begin{align*}
    R(\hat\btheta_{\text{SVM}} ) \leq  \exp\bigg( - \frac{  c_1 n \| \bmu \|_2^4  }{  n  \| \bmu \|_{2}^2+ d } \bigg) = \exp\bigg( - \frac{  c_1 n \gamma^4 s^2  }{  n  \gamma^2 s + d } \bigg),
\end{align*}
where $c_1$ is an absolute constant. This completes the proof. 
\end{proof}

\section{Experiments}
In this section we present simulation results to backup our population risk bound in Theorem~\ref{thm:BOnew}. We generate $\ub$ as a standard Gaussian vector, and set $\bSigma = \diag\{ \lambda_1,\ldots, \lambda_d\}$ with $\lambda_k = k^{-\alpha}$ for some parameter $\alpha \in [0,1)$, which matches the setting studied in Section~\ref{section:mainresults}. The mean vector $\bmu$ is generated uniformly from the sphere centered at the origin with radius $r$. All population risks are calculated by taking the average of $100$ independent experiments. Note that under our setting, $\hat\btheta_{\text{SVM}} = \hat\btheta_{\text{LS}}$ can be easily calculated. Moreover, since we are considering Gaussian mixtures in our experiments, the population risk can be directly calculated with the Gaussian cumulative distribution function:
\begin{align*}
    R(\hat\btheta_{\text{SVM}}) &= \PP[ \btheta^\top \bmu  < y\cdot \hat\btheta_{\text{SVM}}^\top  \bLambda^{1/2} \ub ].
\end{align*}
The derivation of the above result is in the proof of Lemma~\ref{lemma:subGaussian_riskbound} in Appendix~\ref{section:proof_subGaussian_riskbound}.

\noindent\textbf{Population risk versus the norm of the mean vector $\| \bmu\|_2$.} We first present experimental results on the relation between the population risk and the norm of the mean vector $\| \bmu\|_2$. Note that in our setting, the risk bound in Theorem~\ref{thm:BOnew} reduces to the following bound:
\begin{align*}
    R(\hat\btheta_{\text{SVM}}) \leq \exp\Bigg(  \frac{ - C' n \| \bmu \|_2^4  }{  n  \| \bmu \|_{\bSigma}^2+ \sum_{k=1}^d k^{-2\alpha}} \Bigg).
\end{align*}
Based on this bound, we can first see that the population risk should be smaller when $\alpha$ is larger. Moreover, the dependency of $R(\hat\btheta_{\text{SVM}})$ depends on the comparison between the scaling of the two terms in the denominator. When
\begin{align}\label{eq:experimentdiscussion_eq1}
    \sum_{k=1}^d k^{-2\alpha} \geq  n  \| \bmu \|_{\bSigma}^2,
\end{align}
we can expect that $ -\log (R(\hat{\btheta}_{\text{SVM}})) $ should be roughly of order $\| \bmu \|_2^4$. On the other hand, if \eqref{eq:experimentdiscussion_eq1} does not hold, then $ -\log (R(\hat{\btheta}_{\text{SVM}})) $ should be roughly of order $\| \bmu \|_2^2$. It is also clear that whether \eqref{eq:experimentdiscussion_eq1} holds heavily depends on the values of the sample size $n$ and $\alpha$: when $n$ is large, then \eqref{eq:experimentdiscussion_eq1} is less likely to be satisfied. Moreover, when $\alpha > 1/2$, \eqref{eq:experimentdiscussion_eq1} cannot hold because in this case $\sum_{k=1}^d k^{-2\alpha}$ is upper bounded by a constant. 

\begin{figure}[ht!]
	\begin{center}
		\begin{tabular}{cc}
 		\hspace{-0.2in}
			\subfigure[$R(\hat{\btheta}_{\text{SVM}})$ versus $\| \bmu \|_2$, $n = 10$, $d = 2000$]{\includegraphics[width=0.4\linewidth,angle=0]{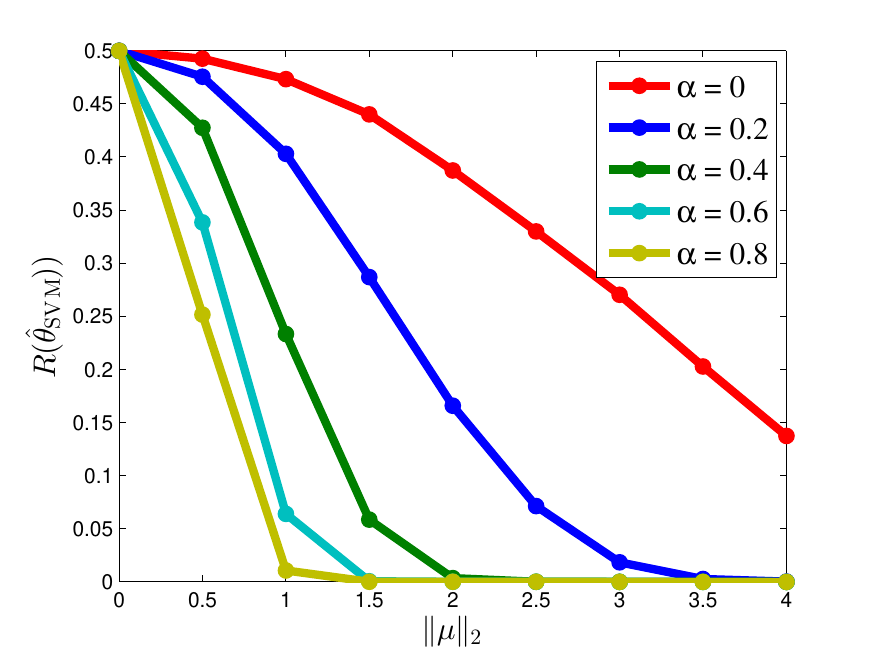}\label{subfig:1}}
			& 
			\subfigure[$ -\log (R(\hat{\btheta}_{\text{SVM}})) $
 versus $\| \bmu \|_2^2$, $n = 10$, $d = 2000$]{\includegraphics[width=0.4\linewidth,angle=0]{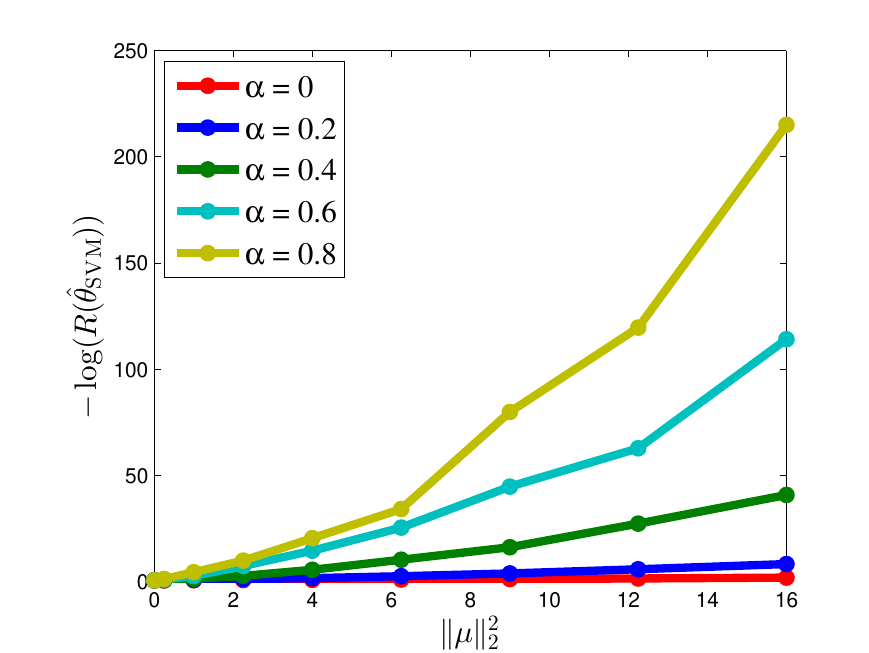}\label{subfig:2}}
			\\
			\subfigure[$R(\hat{\btheta}_{\text{SVM}})$ versus $\| \bmu \|_2$, $n = 100$, $d = 2000$]{\includegraphics[width=0.4\linewidth,angle=0]{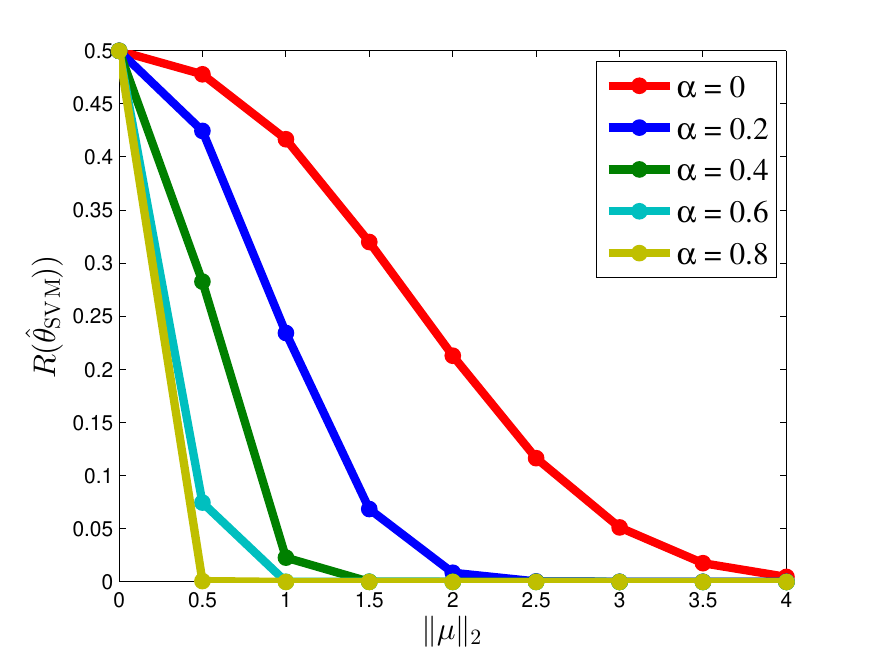}\label{subfig:3}}
			& 
			\subfigure[$-\log(R(\hat{\btheta}_{\text{SVM}}))$ versus $\| \bmu \|_2^2$, $n = 100$, $d = 2000$]{\includegraphics[width=0.4\linewidth,angle=0]{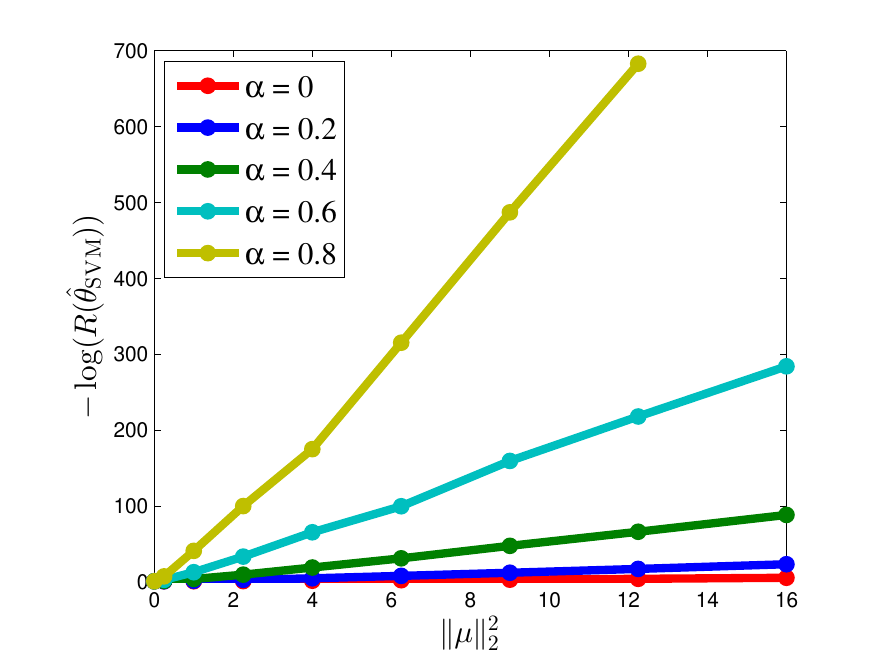}\label{subfig:4}}
		    \end{tabular}
	\end{center}
	\vskip -12pt
	\caption{Experiments on the dependency of the population risk $R(\hat\btheta_{\text{SVM}})$ on the norm of the mean vector $\| \bmu\|_2$ with different values of $\alpha$ and sample size $n$. (a) and (b) gives the curves with $n = 10$, while (c) and (d) are for the case $n = 100$. Moreover, (a) and (c) gives the curves of $R(\hat\btheta_{\text{SVM}})$ versus $\| \bmu\|_2$, and to further test the tightness of our risk bound, in (c) and (d) we also study the relation between $-\log(R(\hat{\btheta}_{\text{SVM}}))$ and $\| \bmu \|_2^2$. The dimension $d$ is set to $2000$ in all these figures. In (d) we omit the last point $\|\bmu \|_2 = 16$ in the curve for $\alpha = 0.8$ because the population risk in this case is too small and is dominated by the numerical accuracy.
	} 
	\label{fig:errorcurve}
\end{figure}

In Figure~\ref{fig:errorcurve}, we verify the above argument by verifying the dependency of the population risk $R(\hat\btheta_{\text{SVM}})$ on the norm of the mean vector $\| \bmu\|_2$ with different values of $\alpha$ and sample size $n$. From Figures~\ref{subfig:1} and \ref{subfig:3}, we can see that $R(\hat{\btheta}_{\text{SVM}})$ decreases with $\| \bmu \|_2$ and $\alpha$. From~\ref{subfig:2}, we verify that when $n = 10$ (which is rather small) and when $\alpha = 0, 0.2, 0.4$, $ -\log (R(\hat{\btheta}_{\text{SVM}})) $ is linear in $\| \bmu \|_2^2$. This verifies our discussion for the setting when \eqref{eq:experimentdiscussion_eq1} holds. On the other hand, when $\alpha = 0.6,0.8$, $ -\log (R(\hat{\btheta}_{\text{SVM}})) $ has a higher order dependency in $\| \bmu \|_2^2$, which is because $\sum_{k=1}^d k^{-2\alpha}$ is upper bounded by a constant and \eqref{eq:experimentdiscussion_eq1} cannot hold. In Figure~\ref{subfig:4}, we further verify that when $n = 100$, \eqref{eq:experimentdiscussion_eq1} never hold and $ -\log (R(\hat{\btheta}_{\text{SVM}})) $ is of order $\| \bmu \|_2^2$ for all choices of $\alpha$. This set of experiments verifies our risk bound in Theorem~\ref{thm:BOnew}.

\noindent\textbf{Verification of the dimension-dependent and dimension-free settings.} In Corollary~\ref{col:polynomialdecay}, we have discussed that when $\alpha < 1/2$, achieving a small population risk requires a larger $\| \bmu \|_2$ when $d$ is larger. On the other hand, when $\alpha > 1/2$, the requirement on $\| \bmu \|_2$ to achieve small population error is dimension-free. Here we present experimental results to verify our claim. The results are given in Figure~\ref{fig:phasetransition}. We can see very clearly that when $\alpha = 0.2$, the risk curves for different $d$ are different, and larger $d$ results in worse population risk. However, when $\alpha = 0.8$, all the risk curves are almost exactly the same, which indicates that the population risk is dimension-free. This verifies our claim in Corollary~\ref{col:polynomialdecay}.

\begin{figure}[h]
	\begin{center}
		\begin{tabular}{cc}
 		\hspace{-0.2in}
			\subfigure[$R(\hat{\btheta}_{\text{SVM}})$ versus $\| \bmu \|_2$, $\alpha = 0.2$, $n = 10$]{\includegraphics[width=0.4\linewidth,angle=0]{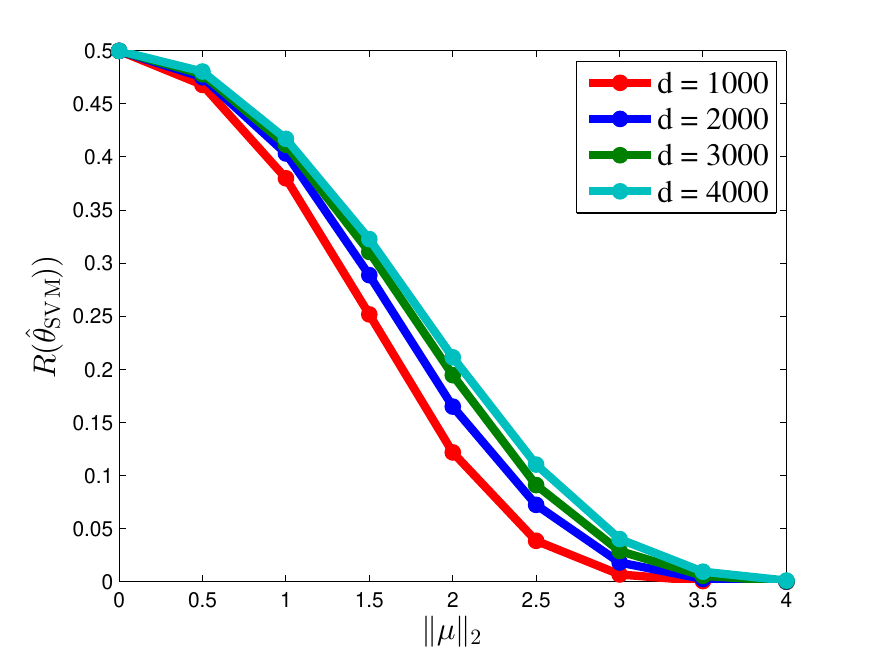}\label{subfig:5}}
			& 
			\subfigure[$R(\hat{\btheta}_{\text{SVM}})$ versus $\| \bmu \|_2$, $\alpha = 0.8$, $n = 10$]{\includegraphics[width=0.4\linewidth,angle=0]{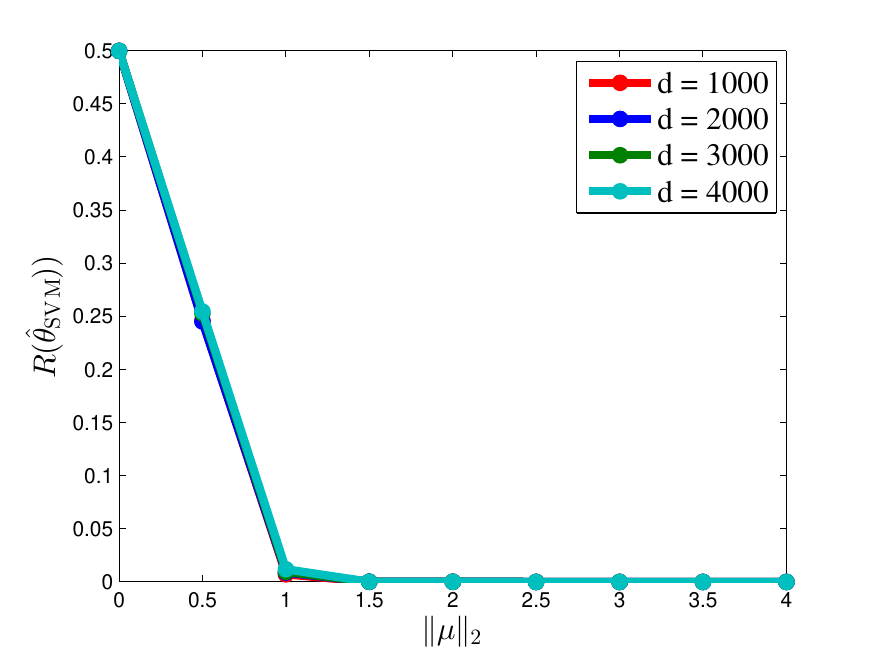}\label{subfig:6}}
		    \end{tabular}
	\end{center}
	\vskip -12pt
	\caption{The population risk curve with respect to $\| \bmu\|_2$ with different values of $\alpha$ and dimension $d$. (a) shows the result for $\alpha = 0.2$, while (b) is for the case $\alpha = 0.8$. The sample size $n$ is set to $10$ in both experiments.} 
	\label{fig:phasetransition}
\end{figure}



\end{document}